\documentclass[11pt]{article}

\usepackage[utf8]{inputenc} % allow utf-8 input
\usepackage[T1]{fontenc}    % use 8-bit T1 fonts
\usepackage{url}            % simple URL typesetting
\usepackage{booktabs}       % professional-quality tables
\usepackage{amsmath,amssymb,amsthm}
\usepackage{amsfonts}
\usepackage{subcaption}
\usepackage{bbm}
\usepackage{breakcites}
\usepackage[colorlinks]{hyperref}
\usepackage[margin=1.0in]{geometry}
\usepackage{mathrsfs}
\usepackage{setspace}
\usepackage{array, makecell}
\usepackage{multirow}
\usepackage{floatrow}
\usepackage{geometry,mathtools}
\usepackage{wrapfig}
\usepackage[ruled,vlined]{algorithm2e}
\usepackage[compact]{titlesec}
\usepackage[compatible]{algpseudocode}
\usepackage{authblk}
\usepackage{cite}

\usepackage{float}
\floatstyle{plaintop}
\restylefloat{table}

\newfloatcommand{capbtabbox}{table}[][\FBwidth]
\usepackage{setspace}

\newtheorem{proposition}{Proposition}
\newtheorem{lemma}{Lemma}
\newtheorem{theorem}{Theorem}

\newtheorem{remark}{Remark}
\newtheorem{definition}{Definition}

\renewcommand{\algorithmiccomment}[1]{\bgroup\hfill\tiny//~#1\egroup}

\usepackage{graphicx}
\usepackage{mathrsfs}
\usepackage{mathtools}
\usepackage{multirow}
\usepackage{xcolor}
\usepackage{hyperref}
\usepackage{subcaption}
\usepackage{soul}
\usepackage{booktabs}       % professional-quality tables
\usepackage[ruled,vlined]{algorithm2e}

\allowdisplaybreaks

\date{}

\begin{document}

\title{\bf Sequential Adversarial Anomaly Detection\\for One-Class Event Data}
\author[1]{Shixiang Zhu}
\affil[1]{Heinz College of Information Systems and Public Policy, Carnegie Mellon University}
\author[2]{Henry Shaowu Yuchi}
\author[2]{Minghe Zhang}
\author[2]{Yao Xie}
\affil[2]{H. Milton Stewart School of Industrial and Systems Engineering, Georgia Institute of Technology.}
\maketitle

\begin{abstract}
We consider the sequential anomaly detection problem in the one-class setting when only the anomalous sequences are available and propose an adversarial sequential detector by solving a minimax problem to find an optimal detector against the worst-case sequences from a generator. The generator captures the dependence in sequential events using the marked point process model. The detector sequentially evaluates the likelihood of a test sequence and compares it with a time-varying threshold, also learned from data through the minimax problem. We demonstrate our proposed method's good performance using numerical experiments on simulations and proprietary large-scale credit card fraud datasets. The proposed method can generally apply to detecting anomalous sequences. 
\end{abstract}


\section{Introduction}

Spatio-temporal event data are ubiquitous nowadays, ranging from electronic transaction records, earthquake activities recorded by seismic sensors and police reports. Such data consist of sequences of discrete events that indicate when and where each event occurred and other additional descriptions such as its category or volume. We are particularly interested in financial transaction fraud, which is often caused by stolen credit or debit card numbers from an unsecured website or due to identity theft. Collected financial transaction fraud typically consists of a series of anomalous events: unauthorized uses of a credit or debit card or similar payment tools (\emph{Automated Clearing House(ACH), Electronic Funds Transfer(EFT)}, recurring charge, et cetera.) to obtain money or property \cite{FBI2021}. As illustrated in Figure~\ref{fig:macy-res1}, such events sequence corresponds to anomalous transaction records,  typically including the time, location, amount, and type of the transactions.

Early detection of financial fraud plays a vital role in preventing further economic loss for involved parties. In today's digital world, credit card fraud and ID theft continue to rise in recent years. Losses to fraud incurred by payment card issuers worldwide reached \$19.21 billion in 2019. Card issuers accounted for 68.97\% of gross fraud losses \cite{Nilson2019} since the liability usually comes down to the merchant or the card issuer, according to the ``zero-liability policies''\footnote{\url{https://usa.visa.com/pay-with-visa/visa-chip-technology-consumers/zero-liability-policy.html}} -- merchants and banks could face a significant risk of economic losses. Credit card fraud also causes much loss and trouble to the customers with stolen identities have been stolen: victims need to report unauthorized charges to the card issuer, canceling the current card, waiting for a new one in the mail, and subbing the new number into all auto-pay accounts linked to the old card. The entire process can take days or even weeks.

% \begin{figure}[!t]
% \centering
% \includegraphics[width=.5\linewidth]{imgs/illustration.pdf}
% \caption{An example of a sequence of anomalous events that are dependent: one leads to another.}
% \label{fig:macy-res1}
% \vspace{-.15in}
% \end{figure}

\begin{figure}[!t]
\centering
\begin{subfigure}[h]{0.47\linewidth}
\includegraphics[width=\linewidth]{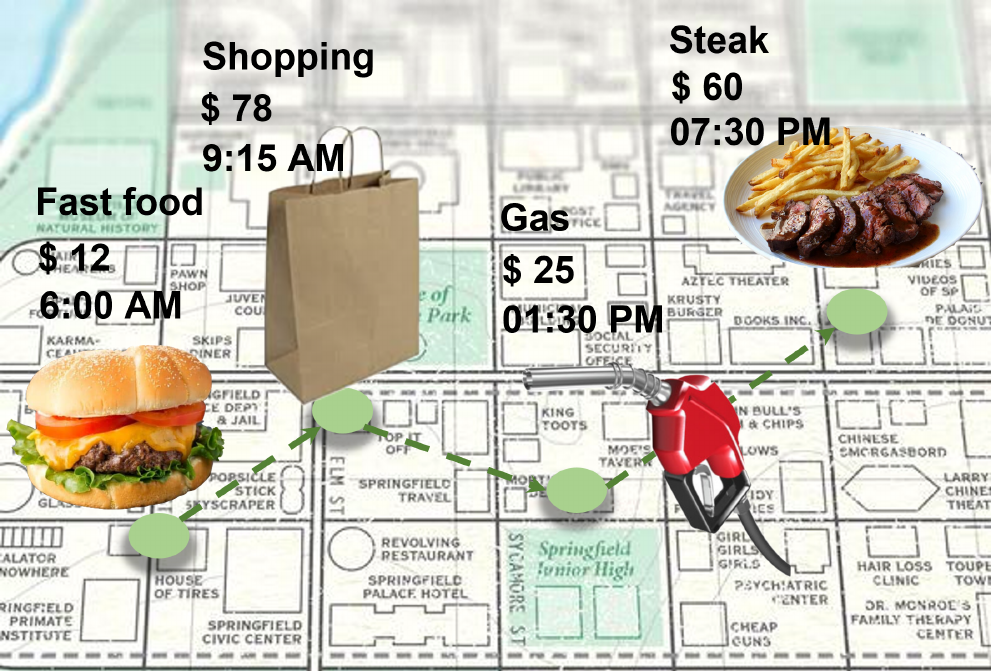}
\end{subfigure}
\hfill
\begin{subfigure}[h]{0.47\linewidth}
\includegraphics[width=\linewidth]{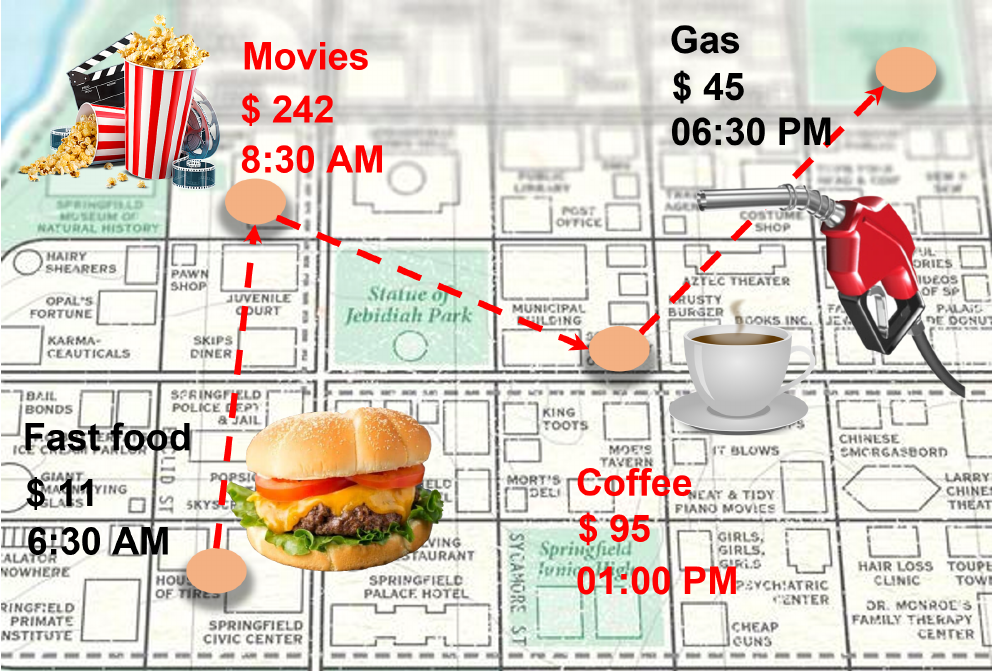}
\end{subfigure}
\vspace{.1in}
\caption{Examples of a sequence of ordinary events (left) and a sequence of anomalous events (right) that are dependent: one leads to another. 
The events on the left show an ordinary pattern of transactions for a consumer. For the events on the right, it is unusual for a consumer to execute a transaction at a movie theater in the early morning and then a high-volume transaction at a coffee shop. It indicates the events are abnormal.}
\label{fig:macy-res1}
\vspace{-.15in}
\end{figure}

% {\bf Practical implication: The challenge posed by one-class.} 
For applications such as financial fraud detection, we usually only have access to the anomalous event sequences. This can be due to protecting consumer privacy, so only fraudulent transaction data are collected for the study. The resulting one-class problem makes the task of anomaly detection even more challenging. Still, there are distinctive patterns of anomalies that enable us to develop powerful detection algorithms. For instance, Figure~\ref{fig:card-demo-1} shows an example of a sequence of fraudulent transactions that we extracted from real data. A fraudster used a stolen card twelve times in just six days and made electronic transactions at stores that are physically far away from each other, ranging from California to New England. The types of transactions are also different from the regular spending pattern. Figure~\ref{fig:card-demo-2} illustrates the distribution of a collection of fraudulent transactions for location (store ID), season, and the number of transactions. We can observe a significant portion of transactions at the department store in San Francisco and New York.

\begin{figure}[!t]
\centering
\begin{subfigure}[h]{0.32\linewidth}
\includegraphics[width=\linewidth]{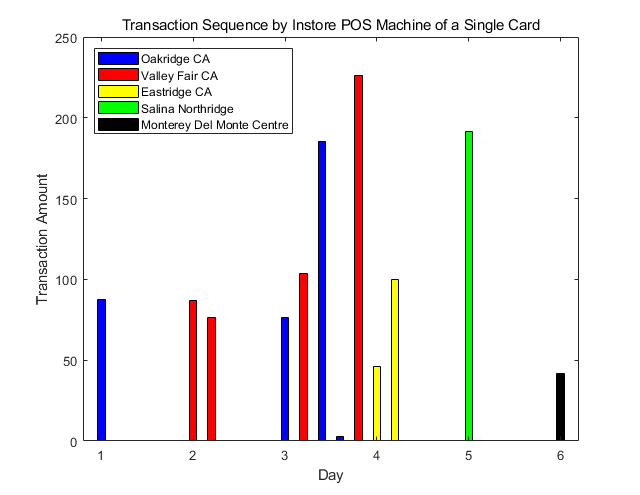}
\caption{transactions by one card}
\label{fig:card-demo-1}
\end{subfigure}
\begin{subfigure}[h]{0.33\linewidth}
\includegraphics[width=\linewidth]{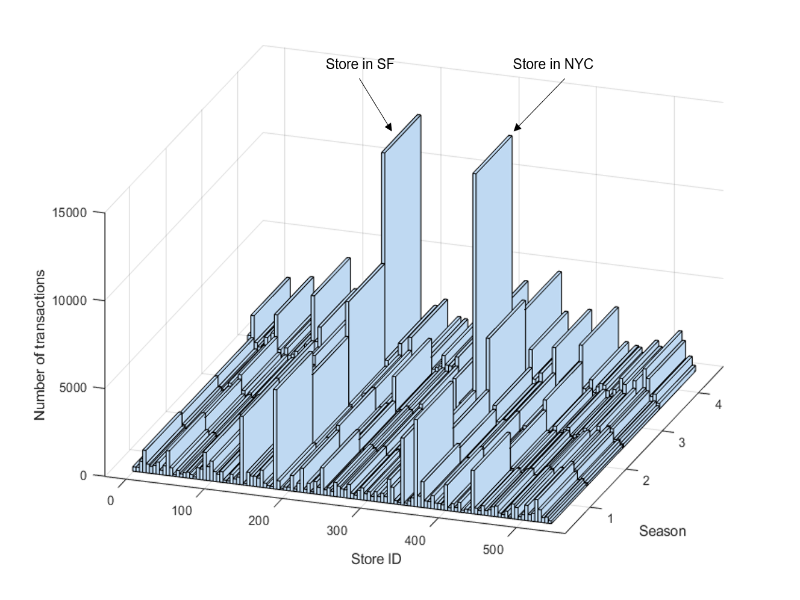}
\caption{locations of the transactions}
\label{fig:card-demo-2}
\end{subfigure}
\begin{subfigure}[h]{0.33\linewidth}
\includegraphics[width=\linewidth]{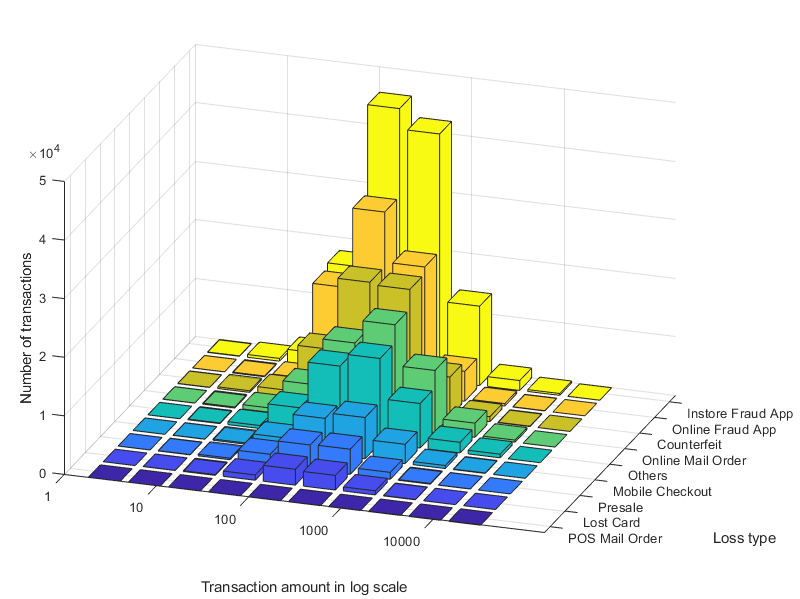}
\caption{types of the transactions}
\label{fig:card-demo-3}
\end{subfigure}
% \begin{subfigure}[h]{0.32\linewidth}
% \includegraphics[width=\linewidth]{imgs/credit-exp-1.png}
% \caption{transactions by one card}
% \label{fig:card-demo-1}
% \end{subfigure}
% \begin{subfigure}[h]{0.33\linewidth}
% \includegraphics[width=\linewidth]{imgs/credit-exp-2.png}
% \caption{locations of the transactions}
% \label{fig:card-demo-2}
% \end{subfigure}
% \begin{subfigure}[h]{0.33\linewidth}
% \includegraphics[width=\linewidth]{imgs/credit-exp-3.png}
% \caption{types of the transactions}
% \label{fig:card-demo-1}
% \end{subfigure}
\vspace{.1in}
\caption{Example of sequential fraud credit card transactions data set provided by a major department store in the US: (a) shows a sequence of transactions made by one stolen credit card; each bar represents a fraudulent transaction, and the bar's height indicates the transaction amount in dollars and the color of the bar indicates the location of the transaction; (b) shows an overview of how these fraudulent transactions were distributed over stores and seasons; (c) shows an overview of how these fraudulent transactions were distributed over the amount of purchase for different loss types.}
\label{fig:card-demo}
\vspace{-.15in}
\end{figure}

Although there has been much research effort in machine learning and statistics for anomaly detection using sequential data \cite{Chandola2010, Xu2010, Chung2015, Doshi2020}, we cannot use existing methods here directly for the following reasons. First, many existing works consider detecting anomalous sequences ``as a whole'' rather than detecting in an online fashion. Second, the one-class data situation requires an unsupervised approach for anomaly detection; however, most sequential anomaly detection algorithms are based on supervised learning. 

This paper presents an adversarial anomaly detection algorithm for one-class sequential detection, where only anomalous data are available. The adversarial sequential detector is solved from a minimax problem to find an optimal detector against the worst-case sequences from a generator that captures the dependence in sequential events using the marked point process model. The detector sequentially evaluates the likelihood of a test sequence and compares it with a time-varying threshold, which is also learned from data through the minimax problem. We demonstrate the proposed method's good performance by comparing state-of-the-art methods on synthetic and proprietary large-scale credit card fraud data provided by a major department store in the US. 

On a high-level, our minimax formulation is inspired by imitation learning \cite{Hussein2017}, which minimizes the maximum mean discrepancy \cite{Gretton2012} (MMD). In particular, the generator is built upon the Long Short-Term Memory (LSTM) \cite{Hochreiter1997} which specifies the conditional distribution of the next event. We parameterize the detector by comparing the likelihood function of marked Hawkes processes with a deep Fourier kernel \cite{Zhu2021, zhu2021neural} with a threshold. The resulted likelihood function is computationally efficient to implement in the online fashion and can capture complex dependence between events in anomalous sequences. A notable feature of our framework is a {\it time-varying} threshold learned from data by solving the minimax problem, which achieves tight control of the false-alarms and hard to obtain precisely in theory. This is a drastic departure from prior approaches in sequential anomaly detection. 

The rest of the paper is organized as follows. We first discuss the related work in sequential anomaly detection and revisit some basic definitions in imitation learning. Section~\ref{sec:anomaly-detection} sets up the problem and introduces our sequential anomaly detection framework. Section~\ref{sec:model} proposes a new marked point process model equipped with a deep Fourier kernel to model-dependent sequential data. Section~\ref{sec:generator} presents the adversarial sequence generator and learning algorithms. Finally, we present our numerical results on both real and synthetic data in Section~\ref{sec:experiments}. Proofs to all propositions can be found in the appendix.

% \vspace{-0.1in}
\subsection{Related work}

Several research lines are related to this work, including imitation learning, the long-term-short-term (LSTM) architecture for modeling sequence data, one-class anomaly detection, and fraud detection, which we review below.

Imitation learning \cite{Hussein2017} aims to mimic the expert's behavior in a given task.
An agent (a learning machine) is trained to perform a task from demonstrations by learning a mapping between observations and actions. Landmark works by \cite{Abbeel2004, Syed2008} attack this problem via \emph{inverse reinforcement learning} (IRL). 
In their work, the learning process is achieved by devising a game-playing procedure involving two opponents in a zero-sum game. This alteration not only allows them to achieve the same goal of doing nearly as well as the expert as in \cite{Abbeel2004} but achieves better performances in various settings. However, this strategy cannot be directly applied to event data modeling without adaptation. 
A recent work \cite{Li2018} filled this gap by introducing a reward function with a non-parametric form, which measures the discrepancy between the training and generated sequences. Their proposed approach models the events using a temporal point process, which draws similarities in our work's spatio-temporal point process model. However, our work differs from \cite{Li2018} in two major ways. Rather than constructing a generative model, we focus on sequential anomaly detection, which is a different type of problem. Besides, we design a structured reward function that is more suitable for modeling the triggering effects between events and more computationally efficient to carry out. 

There is another work on inverse reinforcement learning related to our work. The work in \cite{Ziebart2008} first proposed a probabilistic approach to the imitation learning problem via maximum entropy. The work proposed an efficient state frequency algorithm that is composed of both backward and forward passes recursively. A more recent similar article by \cite{oh2019sequential} seeks to integrate IRL with anomaly detection based on the above maximum entropy IRL framework. They aim to learn the unknown reward function to test a given sequence. The significant difference, however, is they focus on time series data instead of event data. Also, we formulate the problem as a minimax optimization, whereas they used a Bayesian method to estimate the model parameters. 

A large body of recent works performs sequential anomaly detection using LSTM, similar to the proposed stochastic LSTM used as the adversarial generator in our work. In \cite{Malhotra2016}, the authors proposed an encoder-decoder scheme using LSTM to learn the normal behavior of data and used reconstruction errors to find anomalies. The work of \cite{Nanduri2016} built a Recurrent Neural Network (RNN) model with LSTM structure to conduct anomaly detection for multivariate time series data for flight operations. Another paper from \cite{Luo2017} looks into anomaly detection in videos by convolutional neural networks with the LSTM modeling. It is clear that the LSTM model is versatile for various applications and can model unknown complex sequential data. However, the LSTM used here in this paper is stochastic (as a generative model to capture data distribution), whereas most LSTM architectures are deterministic. Specifically, the input at each time step in the stochastic LSTM is drawn from a random variable whose distribution is specified by the LSTM's parameters.

% Another similar work to this paper is from \cite{Zhu2020}, which attempts to conduct one-class anomaly detection using spatio-temporal point process modeling as well. Their approach is different from this work, where the authors utilize the idea of General Adversarial Network (GAN) to generative hypothetical data, subsequently training the model.

% Apart from imitation learning,
There is a wide array of existing research in anomaly detection. Principle Component Analysis (PCA) has traditionally been used to detect outliers, which naturally fits into anomaly detection. In \cite{Chalapathy2017} the authors propose a robust auto-encoder model, which is closely related to PCA for anomaly detection, and a deep neural network is introduced for the training process. Additionally, \cite{Chalapathy2018} proposed a one-class neural network to detect anomalies in complex data sets by creating a tight envelope around the normal data. It is improved from the one-class singular value decomposition formulation to be more robust. A closely related work is \cite{Ruff2018}, which looks into one-class anomaly detection through a deep support vector data description model that finds a data-enclosing hyper-sphere with minimum volume. However, most previous studies on one-class anomaly detection assumed independent and identically distributed data samples, whereas we consider data dependency. 

% {\bf Industrial relevance:}
As an important application of anomaly detection, credit card fraud detection has also drawn a lot of research interest \cite{Bolton2002, Kou2004}. Most commonly, supervised methods have been adopted to use a database of known fraudulent/legitimate cases to construct a model that yields a suspicion score for new cases. Traditional statistical classification methods, such as linear discriminant analysis \cite{Wang2018}, logistic classification \cite{Sahin2011}, and $k$-nearest neighbors \cite{Malini2017}, have proved to be effective tools for many applications. However, more powerful tools \cite{Ghosh1994, Maes2002}, especially neural networks, have also been extensively applied. Unsupervised methods are used when there are no prior sets of legitimate and fraudulent observations. A large body of approaches \cite{Bolton2001, Srivastava2008, Tran2018} employed here are usually a combination of profiling and outlier detection methods, which models a baseline distribution to represent the normal behavior and then detect observations departure from this. Compared to the previous unsupervised studies in credit card fraud detection, the most notable feature of our approach is to learn the fraudulent behaviors by ``mimicking'' the limited amount of anomalies via an adversarial learning framework. Our anomaly detector equipped with the deep Fourier kernel is more flexible than conventional approaches in capturing intricate marked spatio-temporal dynamics between events while being computationally efficient. 

Our work is a significant extension of the previous conference paper \cite{Zhu2020}, which studies the one-class sequential anomaly detection using a framework of Generative Adversarial Network (GAN) \cite{Goodfellow2014-1, Goodfellow2014-2} based on the cross-entropy between the real and generated distributions. Here, we focus on a different loss function motivated by imitation learning and MMD distances that is more computationally efficient. In addition, we introduce a new time-varying threshold, which can be learned in a data-driven manner. 

\section{Background: Inverse reinforcement learning}
\label{sec:background-irl}

Since imitation learning is a form of reinforcement learning (RL), in the following, we will provide some necessary background about RL. Consider an agent interacting with the environment. At each step, the agent selects an action based on its current state, to which the environment responds with a reward value and the next state. 
%
%Denote the state space and the action space as $\mathcal{X}$ and $\mathcal{A}$, respectively. 
%A stationary policy $\pi$ maps each state $x \in \mathcal{X}$ to a probability distribution over the action space $s \in \mathcal{S}$.
%
The \emph{return} is the sum of (discounted) rewards through the agent's trajectory of interactions with the environment. 
The \emph{value function} of a policy describes the expected \emph{return} from taking action from a state. 
The \emph{inverse reinforcement learning} (IRL) aims to find a reward function from the expert demonstrations explaining the expert behavior. 
Seminal works \cite{Abbeel2004, Ng2000} provide a max-min formulation to address the problem. The authors propose a strategy to match an observed expert policy's value function and a learner's behavior.
Let $\pi$ denote the expert policy, and $\pi_\varphi$ denote the learner policy, respectively.
The optimal reward function $r$ can be found as the saddle-point of the following max-min problem \cite{Syed2008}, i.e.,
\begin{equation*}
\begin{aligned}
    \underset{r \in \mathcal{F}}{\max}~\underset{\varphi \in \mathcal{G}}{\min}~\Bigg\{ 
    & \mathbb{E}_{\boldsymbol{x} \sim \pi} \left[ \sum_{i=1}^{N_x} r(x_i, s_i) \right] - \mathbb{E}_{\boldsymbol{z} \sim \pi_\varphi} \left[ \sum_{j=1}^{N_z} r(z_j, s_j) \right] \Bigg\}, \\
    \label{eq:def-reward}
\end{aligned}
\end{equation*}
where $\mathcal{F}$ is the family class for reward function and $\mathcal{G}$ is the family class for learner policy. Here, $\boldsymbol{x} = \{x_1, \dots, x_{N_x}\}$ is a sequence of actions generated by the expert policy $\pi$,  $\boldsymbol{z} = \{z_1, \dots, z_{N_z}\}$ is a roll-out sequence generated from the learner policy $\pi_\varphi$, and $N_x$ and $N_z$ are the numbers of actions for sequences $\boldsymbol{x}$ and $\boldsymbol{z}$, respectively. 
The formulation means that a proper reward function should provide the expert policy a higher reward than any other learner policy in $\mathcal{G}$. The learner can also approach the expert performance by maximizing this reward.

%Formulating the exact form of reward function is an ill-posed problem. A few previous works \cite{Abbeel2004, Ratliff2006} have suggested a linear reward function with respect to the feature of the state. However, this method suffers from some significant drawbacks when no single reward function makes demonstrated behaviour both optimal and significantly better than any alternative behaviour. This occurs quite frequently when, for instance, the behaviour demonstrated by the expert is imperfect or incomplete. In such a case, the planning algorithm only captures a part of the relevant state space and cannot describe the observed behaviour correctly. 

% \vspace{-.1in}
\section{Adversarial sequential anomaly detection}
\label{sec:anomaly-detection}

% We focus on a setting where only anomalous sequences are available. 
% We aim to develop a detector that can detect the anomalous sequence in an online fashion and detect the anomaly as soon as possible. 
We aim to develop an algorithm to detect anomalous sequences when the training dataset consists of only abnormal sequences and without normal sequences. In particular, the algorithm will process data sequentially and raise the alarm as soon as possible after the sequence has been identified as anomalous. 
Denote such a detector as $\ell$ with parameter $\theta$. At each time $t$, the detector evaluates a statistic and compares it with a threshold. For a length-$N$ sequence $\boldsymbol{x}$, define $\boldsymbol x_{1:i} := [x_1, \ldots, x_i]^\top$, $i = 1, 2, \ldots, N$ be its first $i$ observations. We define the detector as a stopping rule, which stops and raises the alarm the first time that the detection statistic exceeds the threshold:
\[
    % \begin{equation}
       T = \inf \{t: \ell(\boldsymbol{x}_{1:i}; \theta) > \eta_i,~t_i \le t < t_{i+1} \} , \quad i = 1, \ldots, N.
        \label{eq:detector}
    % \end{equation}
\]
Once an alarm is raised, the sequence is flagged as an anomaly. If there is no alarm raised till the end of the time horizon, the sequence is considered normal. Note that the test sequence can be an arbitrary (finite) length.

%\vspace{-0.1in}
\subsection{Proposed: Adversarial anomaly detection} 

% Now we describe our adversarial anomaly detector.
Assume a set of anomalous sequences drawn from an empirical distribution $\pi$. Since normal sequences are not available, we introduce an \emph{adversarial generator}, which produces ``normal'' sequences that are statistically similar to the real anomalous sequences. The detector has to discriminate the true anomalous sequence from the counterfeit ``normal'' sequences.  
We introduce competition between the anomaly detector and the generator to drive both models to improve their performances until anomalies can distinguish from the worst-case counterfeits. We can also view this approach as finding the ``worst-case'' distribution that defines the ``border region'' for detection. 
Formally, we formulate this as a minimax problem as follows:
\begin{equation}
    \underset{\varphi \in \mathcal{G}}{\min}~\underset{\theta \in \Theta}{\max}~J(\theta, \varphi) \coloneqq 
    \mathbb{E}_{\boldsymbol{x}\sim\pi}~\ell(\boldsymbol{x};\theta) - \mathbb{E}_{\boldsymbol{z}\sim G_z(\varphi)}~\ell(\boldsymbol{z};\theta),
    \label{eq:adv-imit-learning}
\end{equation}
where $G_z$ is an adversarial generator specified by the parameter $\varphi \in \mathcal{G}$ and $\mathcal{G}$ is a family of candidate generators. Here the detection statistic corresponds to $\ell(\theta)$, the log-likelihood function of the sequence specified by $\theta \in \Theta$ and $\Theta$ is its parameter space. The choices of the adversarial generator and the detector are further discussed in Section~\ref{sec:model}. The detector compares the detection statistic to a threshold. We define the following:
% \vspace{-0.1in}
\begin{definition}[Adversarial sequential anomaly detector]
% \vspace{-.05in}
    Denote the solution to the minimax problem \eqref{eq:adv-imit-learning} as $(\theta^*, \varphi^*)$. A sequential adversarial detector raises an alarm at the time $i$ if
    % \begin{equation}
    \[
        \ell(\boldsymbol{x}_{1:i}; \theta^*) > \eta_i^*,
        % \label{eq:detector}
    % \end{equation}
    \]
    where the time-varying threshold $\eta_i^* \propto \mathbb{E}_{\boldsymbol{z} \sim G_z(\varphi^*)}~\ell(\boldsymbol{z}_{1:i};\theta^*)$.
% \vspace{-.05in}
\label{def:detector}
\end{definition}

%\vspace{-0.1in}
\subsection{Time-varying threshold}

\begin{figure}[!t]
\centering
\includegraphics[width=.4\linewidth]{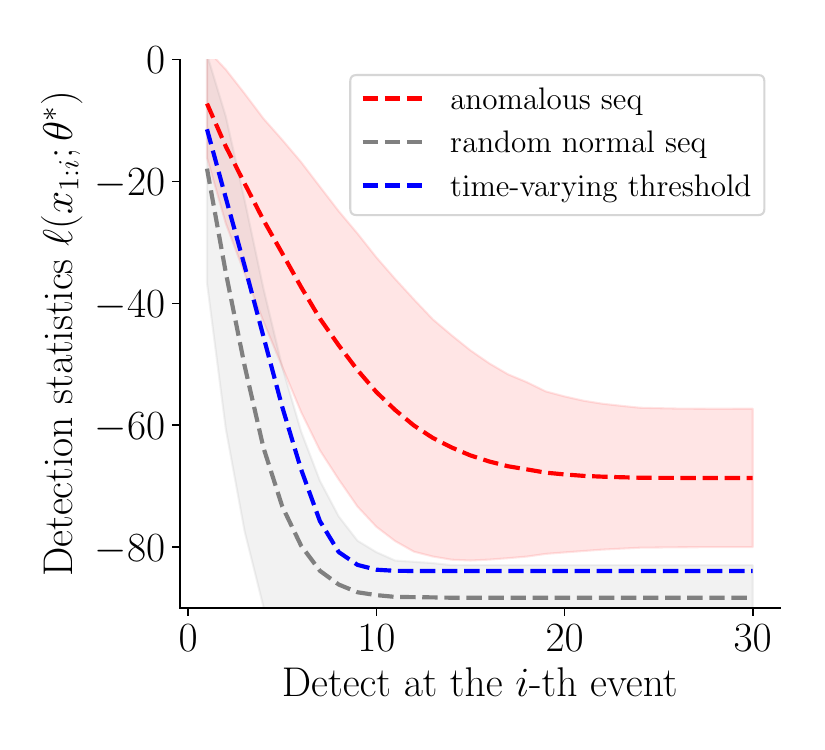}
\caption{An example of the adversarial anomaly detection and threshold using synthetic data with 1000 synthetic sequences. The \emph{two lighter dashed} lines represent mean detection statistics ($\ell(\theta^*)$) for anomalous and normal sequences. The \emph{dashed line in the middle} corresponds to the time-varying threshold suggested by our model. Clearly, the threshold can separate the anomalous sequences from the normal sequences.}
\label{fig:macy-res}
\vspace{-.15in}
\end{figure}

We choose the time-varying threshold $\eta_i^* \propto \mathbb{E}_{\boldsymbol{z} \sim G_z(\varphi^*)}~\ell(\boldsymbol{z}_{1:i};\theta^*)$. Since the value of log-likelihood function $\ell(x_{1:i}; \theta^*)$ for partial sequence observation $\boldsymbol{x}_{1:i}$ may vary over the time step $i$ (the $i$-the event is occurred), we need to adjust the threshold accordingly for making decisions as a function of $i$. Note that our time-varying threshold $\eta_i^*$ is different from sequential statistical analysis, where the threshold for performing detection is usually constant or pre-set (not adaptive to data) based on the known distributions of the data sequence (e.g., set the threshold growing over time as $\sqrt t$ \cite{Siegmund1985}). The rationale behind the design of the threshold $\eta_i^*$ is that, at any given time step, the log-likelihood of the data sequence is larger than that of the generated adversarial sequence; therefore, $\eta_i^*$ provides the tight lower bound for the likelihood of anomalous sequences $\ell(\boldsymbol{x};\theta^*)$ due to the minimization in \eqref{eq:adv-imit-learning}. That is, for any $\varphi \in \mathcal{G}$, 
\begin{align*}
0 \le &~\mathbb{E}_{\boldsymbol{x}\sim\pi}~\ell(\boldsymbol{x}_{1:i};\theta^*) - \eta_i^* \\
\le &~\mathbb{E}_{\boldsymbol{x}\sim\pi}~\ell(\boldsymbol{x}_{1:i};\theta^*) - \mathbb{E}_{\boldsymbol{z}\sim G_z(\varphi)}~\ell(\boldsymbol{z}_{1:i};\theta^*).
\end{align*}
The adversarial sequences drawn from $G_z(\varphi^*)$ can be viewed as the normal sequences that are statistically ``closest'' to anomalous sequences. Therefore, the log-likelihood of such sequences in the ``worst-case'' scenario defines the ``border region'' for detection. In practice, the threshold $\eta_i^*$ can be estimated by $1/n^\prime \sum_{l=1}^{n^\prime} \ell(\boldsymbol{z}^l_{1:i}; \theta^*)$, where $\{\boldsymbol{z}^l\}_{l=1,\dots,n'}$ are adversarial sequences sampled from $G_z(\varphi)$  and $n^\prime$ is the number of the sequences.
As a real example presented in Figure~\ref{fig:macy-res}, the time-varying threshold in the \emph{darker dashed} line can sharply separate the anomalous sequences from the normal sequences. More experimental results are presented in Section~\ref{sec:experiments}.

\begin{figure}[!t]
\centering
\includegraphics[width=.7\linewidth]{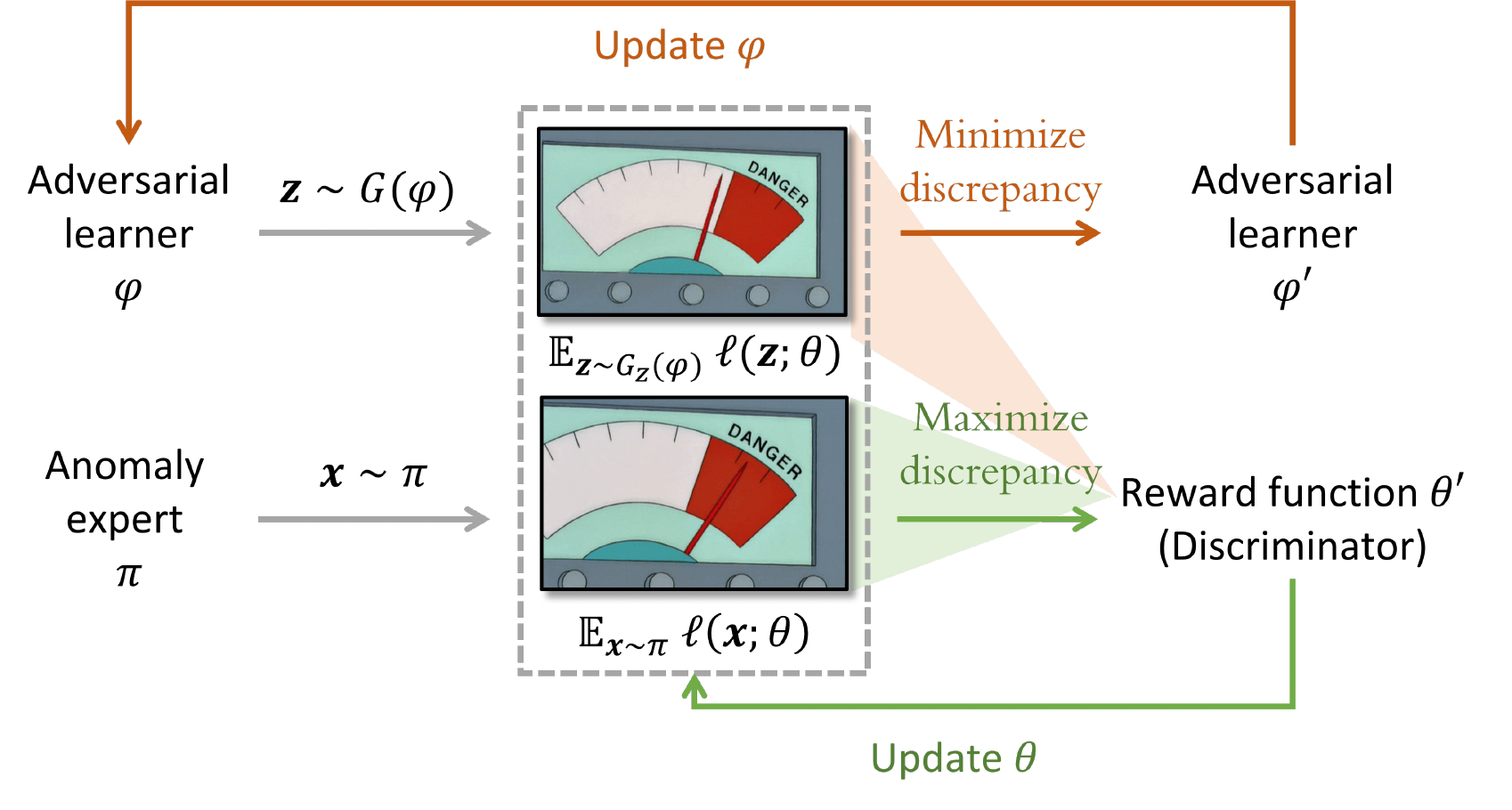}
\caption{An illustration of the imitation learning interpretation. }
\label{fig:illustration-il}
\vspace{-.15in}
\end{figure}

%\vspace{-0.1in}
\subsection{Connection to imitation learning} 

%Our framework can also be viewed as an instance of imitation learning \cite{Abbeel2004}.
The problem formulation \eqref{eq:adv-imit-learning} resembles the minimax formulation in inverse reinforcement learning (IRL) proposed by seminal works \cite{Abbeel2004, Ng2000}. 
As shown in Figure~\ref{fig:illustration-il}, an observed anomalous samples $\boldsymbol{x} \sim \pi$ can be regarded as an expert demonstration sampled from the expert policy $\pi$, where each $\boldsymbol x = \{x_1, \dots, x_{N}\}$ is a sequence of events with length of $N$ and the sequences may be of different lengths. Each event $x_i, i=1,\dots,N$ of the sequence is analogous to the $i$-th action made by the expert given the history of past events $\{x_1, x_2, \dots, x_{i-1}\}$ as the corresponding state.
Accordingly, the generator can be regarded as a learner that generates convincing counterfeit trajectories.

The log-likelihood of observed sequences can be interpreted as undiscounted \emph{return}, i.e., the accumulated sum of rewards evaluated at past actions, where the logarithm of the conditional probability of each event (action) can be regarded as the event's reward. 
The ultimate goal of the proposed framework \eqref{eq:adv-imit-learning} is to close the gap between the return of the expert demonstrations and the return of the learner trajectories so that the counterfeit trajectories can meet the lower bound of the real demonstrations.
% so that the ``quality'' of the counterfeit trajectories can meet minimum standards of the real demonstrations.
% Other than previous imitation learning frameworks aiming to obtain the optimal learner policy, our goal is to find the optimal return function that best discriminates the anomalies from normal data.

%\vspace{-.1in}
\subsection{Connection to MMD-like distance} 

The proposed approach can also be viewed as minimizing a maximum mean discrepancy (MMD)-like distance metric \cite{Gretton2012} as illustrated in Figure~\ref{fig:illustration-mmd}. 
More specifically, the maximization in \eqref{eq:adv-imit-learning} is analogous to an MMD metric in a reduced function class specified by $\Theta$, i.e.,
$
    \sup_{\theta \in \Theta}~\mathbb{E}_{\boldsymbol{x}\sim\pi}~\ell(\boldsymbol{x}; \theta) - \mathbb{E}_{\boldsymbol{z}\sim \varphi}~\ell(\boldsymbol{z};\theta)
$,
where $\Theta$ may not necessarily be a space of continuous, bounded functions on sample space. As shown in \cite{Gretton2012}, if $\Theta$ is sufficiently expressive (universal), e.g., the function class on reproducing kernel Hilbert space (RKHS), then maximization over such $\Theta$ is equivalent to the original definition. Based on this, we select a function class that serves our purpose for anomaly detection (characterizing the sequence's log-likelihood function), which has enough expressive power for our purposes. Therefore, the problem defined in \eqref{eq:adv-imit-learning} can be regarded as minimizing such an MMD-like metric between the empirical distribution of anomalous sequences and the distribution of adversarial sequences.
The minimal MMD distance corresponds to the best ``detection radius'' that we can find without observing normal sequences.

\begin{figure}[!t]
\centering
\includegraphics[width=.5\linewidth]{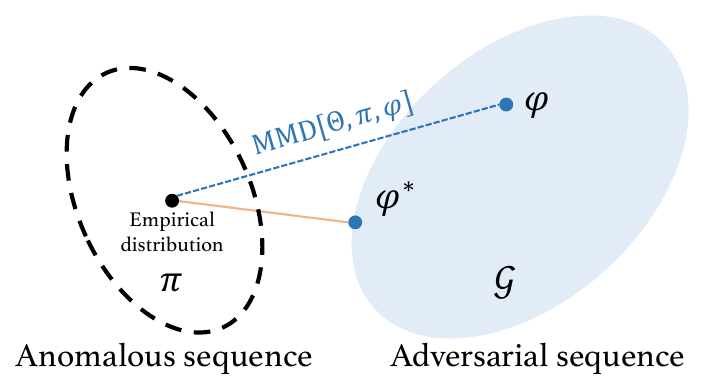}
\caption{The empirical distribution of anomalous sequences is $\pi$. The assumed family of candidate generators is $\mathcal{G}$. Our proposed framework aims to minimize the maximum mean discrepancy (MMD) in a reduced function class $\Theta$ between $\pi$ and $\varphi \in \mathcal{G}$.}
\label{fig:illustration-mmd}
\vspace{-.15in}
\end{figure}

\section{Point process with deep Fourier kernels}
\label{sec:model}

In this section, we present a model for the discrete events, which will lead to the detection statistic (i.e., the form of the likelihood function $\ell(\boldsymbol{x};\theta)$). 
We present a marked Hawkes process model that captures marked spatio-temporal dynamics between events.
%consist of time, location, and marks correlated. 
The most salient feature of the model is that we develop a novel deep Fourier kernel for Hawkes process (c.f. Section 6.6 \cite{Mohri2012} for discussion of Fourier kernel), where
the deep Fourier kernel empowers the model to characterize the intricate non-linear dependence between events while enabling efficient computation of the likelihood function by leading to a closed-form expression of an integral in the likelihood function -- a notorious difficulty in evaluating the likelihood function for Hawkes processes. 

% First, we further assume
Assume each observation is a {\it marked spatio-temporal tuple} which consists of time, location, and marks: $x_i = (t_i, m_i)$, where $t_i \in [0, T)$ is the time of occurrence of the $i$th event, and $m_i \in \mathcal{M} \subseteq \mathbb{R}^d$ is the $d$-dimensional mark (here we treat location as one of a mark). The event's time is important because it defines the event's order and the time interval, which carries the key information.

%\vspace{-0.1in}
\subsection{Preliminary: Marked temporal point processes}
\label{sec:background-pp}

The marked temporal point processes (MTPPs) \cite{Hawkes1971,Reinhart2017} offer a versatile mathematical framework for modeling sequential data consisting of an ordered sequence of discrete events localized in time and mark spaces (space or other additional information). They have proven useful in a wide range of applications \cite{Embrechts1997, Clifton2011, Luca2014, Rambaldi2016, Li2017}.
% Estimating the triggering effects is a fundamental problem for modeling sequential discrete events.
% A default way of expressing the triggering function is through the use of simply defined parametric functions as an example shown in Section~\ref{sec:background-pp}.
% However, for some applications, these parametric assumptions may be limited to describe some complex triggering effects. 
Recent works \cite{Du2016, Mei2017, Li2018, Upadhyay2018, Xiao2017A, Xiao2017B, Zhu2020} have achieved many successes in modeling temporal event data (some with marks) correlated in the time domain using Recurrent Neural Networks (RNNs).

Let $\{x_1 , x_2 , \dots , x_{N_T}\}$ represent a sequence of observations. Denote $N_T$ as the number of the points generated in the time horizon $[0, T)$. The events' distributions in MTPPs are characterized via a conditional intensity function $\lambda(t, m|\mathcal{H}_t)$, which is the probability of observing an event in the marked temporal space $[0, T) \times \mathcal{M}$ given the events' history $\mathcal{H}_t = \{ (t_i, m_i)|t_i < t \}$, i.e.,
\begin{equation}
    \lambda(t, m | \mathcal{H}_t)
    = \frac{\mathbb{E}\left[ N([t, t+dt) \times B(m, dm)) | \mathcal{H}_t \right]}{|B(m, dm)| dt},
    \label{eq:def-conditional-intensity}
\end{equation}
where $N(A)$ is the counting measure of events over the set $A \subseteq [0, T) \times \mathcal{M}$ and $|B(m, dm)|$ is the Lebesgue measure of the ball $B(m, dm)$ centered at $m$ with radius $dm$. 
%
%For instance, as a type of self-exciting point processes, Hawkes processes \cite{Hawkes1971} have been widely used to capture the mutual excitation dynamics among temporal events. 
Assuming that influence from past events are linearly additive for the current event, the conditional intensity function of a Hawkes process is defined as 
\begin{equation}
    \lambda(t, m | \mathcal{H}_{t}) = \mu + \sum_{t_i < t} g(t - t_i, m - m_i),
    \label{eq:hawkes}
\end{equation}
where $\mu \ge 0$ is the background intensity of events, $g(\cdot, \cdot) \ge 0$ is the \emph{triggering function} that captures spatio-temporal and marked dependencies of the past events. The triggering function can be chosen in advance, e.g., in one-dimensional cases, $g(t-t_i) = \alpha \exp \{ - \beta (t - t_i) \}$.

Let $t_n$ denote the last occurred event before time $t$. The conditional probability density function of a point process is defined as 
% \begin{align*}
% \begin{equation}
\[
    f(t,m|\mathcal{H}_t) = \lambda(t,m|\mathcal{H}_t)
    \exp\left\{ - \int_{t_n}^t \int_\mathcal{M} \lambda(t^\prime, m^\prime|\mathcal{H}_{t^\prime}) dm^\prime dt^\prime \right\}.
\]
%     \label{eq:pp-cond-pdf}
% \end{equation}
% \end{align*}
%thanSee Appendix~\ref{append:conditional-intensity} for the derivation for the form of the conditional intensity. 
The log-likelihood of observing a sequence with $N_T$ events denoted as $\boldsymbol{x} = \{(t_i, m_i)\}_{i=1}^{N_T}$ can be obtained by:
\begin{equation}
    \ell(\boldsymbol{x}; \theta) = \sum_{i=1}^{N_T} \log \lambda(t_i, m_i|\mathcal{H}_{t_i}) - \int_0^T \int_\mathcal{M} \lambda(t, m|\mathcal{H}_t) dm dt.
    \label{eq:pp-loglik}
\end{equation}
%See Appendix~\ref{append:Likelihood} for derivation of the log-likelihood function.

% \vspace{-0.1in}
\subsection{Hawkes processes with deep Fourier kernel}
\label{sec:pp-deep-fourier-kernel}

One major computational challenge in evaluating the log-likelihood function is the computation of the integral in (\ref{eq:pp-loglik}), which is multi-dimensional and performed in the possibly continuous mark and time-space. It can be intractable for a general model without a carefully crafted structure. 

To tackle this challenge, we adopt an approach to represent the Hawkes process's triggering function via a Fourier kernel. The Fourier features spectrum is parameterized by a deep neural network, as shown in Figure~\ref{fig:score-illustration}.
For the sake of notational simplicity, we denote $x \coloneqq (t, m) \in \mathcal{X}$ as the most recent event and $x' \coloneqq (t', m') \in \mathcal{X},~t' < t$ as an occurred event in the past, where $\mathcal{X} \coloneqq [0, T] \times \mathcal{M} \subset \mathbb{R}^{d+1}$ is the space for time and mark. Define the conditional intensity function as
\begin{equation}
    \lambda(x|\mathcal{H}_{t};\theta) = \mu + \alpha \sum_{t' < t}  K(x, x'),
    \label{eq:lambda-fourier-kernel}
\end{equation}
where $\alpha$ represents the magnitude of the influence from the past, $\mu \ge 0$ is the background intensity of events. The kernel function $K(x, x^\prime)$ measures the influence of the past event on the current event $x, x^\prime \in \mathcal{X}$, and we will parameterize its kernel-induced feature mapping using a deep neural network $\theta \in \Theta$. 
% In the remainder of this section, we first introduce a general definition of deep Fourier kernel; then we present a deep neural network that embodies the optimal spectrum of Fourier features; last but not least, we demonstrate an efficient log-likelihood evaluation for this point process model, which will be used as our return value in \eqref{eq:objective}.

\begin{figure}[!t]
\centering
\includegraphics[width=.7\linewidth]{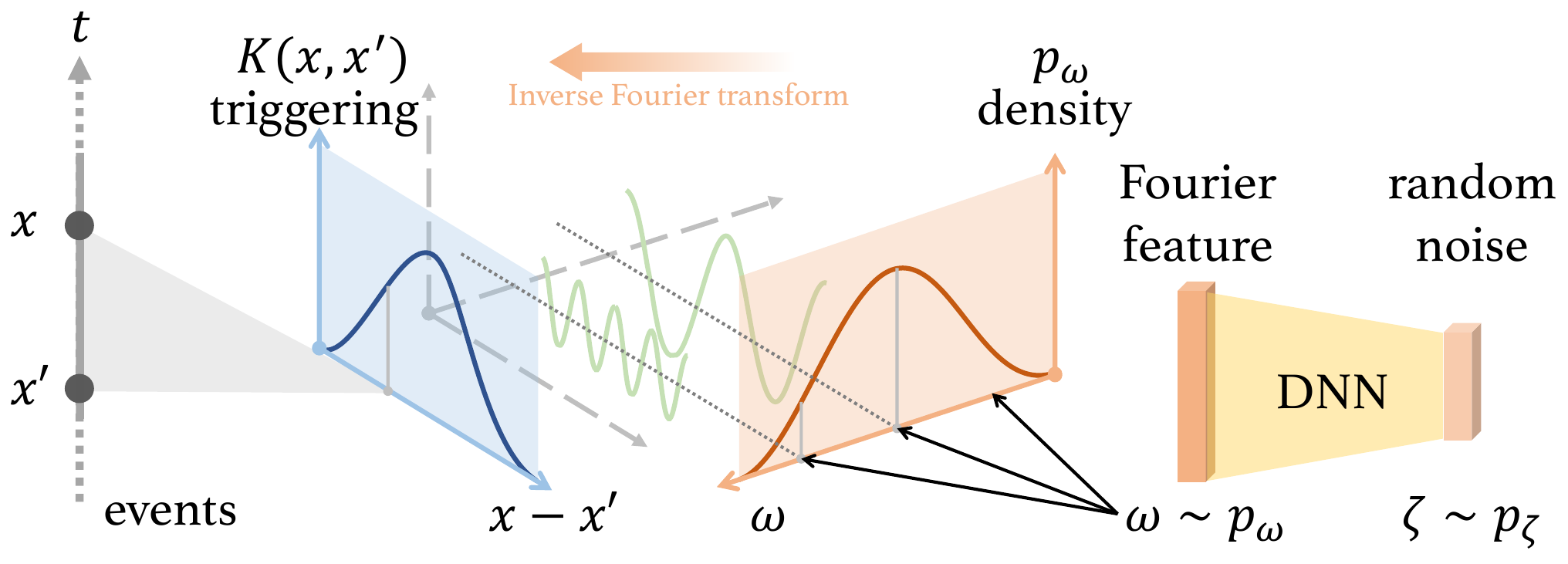}
\caption{An illustration of the Fourier kernel function $K(x, x')$ and its Fourier representation; a deep neural network represents the spectrum of Fourier features.}
\label{fig:score-illustration}
% \vspace{-.1in}
\end{figure}

%\subsection{Deep Fourier kernel function.} 
The formulation of deep Fourier kernel function relies on Bochner’s Theorem \cite{Rudin1962}, which states that any bounded, continuous, and shift-invariant kernel is a Fourier transform of a bounded non-negative measure:
\begin{theorem}[Bochner \cite{Rudin1962}]
\label{thm1}
A continuous kernel of the form $K(x,x^\prime)=g(x-x^\prime)$ defined over a locally compact set $\mathcal{X}$ is positive definite if and only if $g$ is the Fourier transform of a non-negative measure:
\begin{equation}
    K(x, x^\prime) = g(x - x^\prime) =\int_{\Omega}p_\omega(\omega)e^{\boldsymbol{i}w^\top (x - x^\prime)}d\omega,
    \label{eq:bochner-thm}
\end{equation}
where $\boldsymbol{i} = \sqrt{-1}$, $p_\omega$ is a non-negative measure, $\Omega$ is the Fourier feature space, and kernels of the form $K(x,x^\prime)$ are called shift-invariant kernel. 
\end{theorem}

If a shift-invariant kernel in (\ref{eq:bochner-thm}) is positive semi-definite and scaled such that $g(0) = 1$, Bochner’s theorem ensures that its Fourier transform $p_\omega$ can be viewed as a probability distribution function since it normalize to 1 and is non-negative. In this sense, the spectrum $p_\omega$ can be viewed as the distribution of $r$-dimensional Fourier features indexed by $\omega \in \Omega \subset \mathbb{R}^r$. 
Hence, we may obtain a triggering function in \eqref{eq:lambda-fourier-kernel} between two events $x, x^\prime \in \mathcal{X} \subset \mathbb{R}^{d+1}$ which satisfies the ``kernel embedding'': %that satisfies these conditions as the following proposition:
\begin{proposition}
Let the triggering function $K$ be a continuous real-valued shift-invariant kernel and $p_\omega$ a probability distribution function. Then
\begin{equation}
    K(x, x^\prime) \coloneqq \mathbb{E}_{\omega\sim p_\omega}
    \big [ \phi_{\omega}(x) \cdot \phi_{\omega}(x^\prime) \big ],
    \label{eq:triggering-func}
\end{equation}
where $\phi_{\omega}(x) \coloneqq \sqrt{2} \cos(\omega^\top W x + u)$ and $W \in \mathbb{R}^{r \times (d+1)}$ is a weight matrix. 
These Fourier features $\omega \in \Omega \subset \mathbb{R}^r$ are sampled from $p_{\omega}$ and $u$ is drawn uniformly from $[0, 2\pi]$.
\label{prop:triggering-func}
\end{proposition}
%The proof is given in Appendix~\ref{append:proof-prop-1}.

In practice, the expression \eqref{eq:triggering-func} can be approximated empirically, i.e.,
\begin{equation}
    \widetilde K(x, x') = \frac{1}{D} \sum_{k=1}^D \phi_{\omega_k}(x) \cdot \phi_{\omega_k}(x^\prime) = \Phi(x)^\top \Phi(x^\prime),
    \label{eq:triggering-func-estimation}
\end{equation}
where $\omega_k,~k=1,\dots,D$ are $D$ Fourier features sampled from the distribution $p_\omega$.
The vector $\Phi(x) \coloneqq [ \phi_{\omega_1}(x), \dots, \phi_{\omega_D}(x) ]^\top$ can be viewed as the approximation of the kernel-induced feature mapping for the score. In the experiments, we substitute $\exp\{\boldsymbol{i}w^\top (x - x^\prime)\}$ with a real-valued feature mapping, such that the probability distribution $p_\omega$ and the kernel $K$ are real \cite{Rahimi2008}.

The next proposition shows the empirical estimation (\ref{eq:triggering-func-estimation}) converges to the population value uniformly over all points in a compact domain $\mathcal{X}$ as the sample size $D$ grows. It is a lower variance approximation to \eqref{eq:triggering-func}.
\begin{proposition}
Assume $\sigma_p^2 = \mathbb{E}_{\omega \sim p_\omega} [\omega^\top \omega] < \infty$ and a compact set $\mathcal{X} \subset \mathbb{R}^{d+1}$. Let $R$ denote the radius of the Euclidean ball containing $\mathcal{X}$. Then for the kernel-induced feature mapping $\Phi$ defined in \eqref{eq:triggering-func-estimation}, we have
\begin{equation}
% \begin{aligned}
    \mathbb{P}\left\{\underset{x, x^\prime \in \mathcal{X}}{\sup} \left | \Phi(x)^\top \Phi(x^\prime) - K(x, x^\prime) \right | \ge \epsilon \right\}
    \le~\left(\frac{48 R \sigma_p}{\epsilon}\right)^2\exp\left\{ - \frac{D \epsilon^2}{4(d+3)} \right\}.
    \label{eq:convergence}
% \end{aligned}
\end{equation} 
\label{prop:triggering-func-convergence}
\end{proposition}
%The proof is given in Appendix~\ref{append:proof-prop-2}. 
The proposition ensures that kernel function can be consistently estimated using a finite number of Fourier features. In particular, note that for an error bound $\epsilon$, the number of samples needed is on the order of $D = O((d+1)\log(R\sigma_p/\epsilon) / \epsilon^2)$, which grows linearly as data dimension $d$ increases, implying the sample complexity is mild in the high-dimensional setting.

\begin{figure*}[!t]
\centering
\includegraphics[width=.75\linewidth]{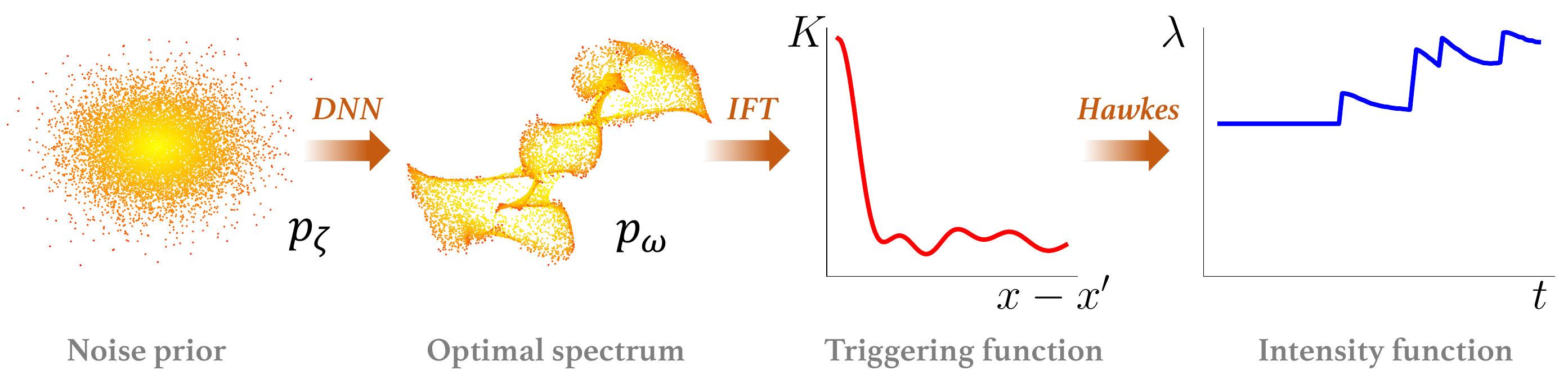}
\caption{An instance of calculating the conditional intensity $\lambda$ through performing inverse Fourier transform: (1) generate random noise; (2) map the noises to the frequencies according to the optimal spectrum; (3) performs inverse Fourier transform (IFT) in the frequency domain and obtain the triggering function; (4) calculate the intensity function based on the triggering function.}
\label{fig:fourier-kernel-exp}
% \vspace{-.15in}
\end{figure*}

%\paragraph{Fourier feature generator.} 
To represent the distribution $p_{\omega}$, we assume it is a transformation of random noise $\zeta \sim p_\zeta$ through a non-linear mapping $\psi_0: \mathbb{R}^q \rightarrow \mathbb{R}^r$, as shown in Figure~\ref{fig:score-illustration}, where $\psi_0$ is a differentiable and it is represented by a deep neural network and $q$ is the dimension of the noise. Roughly speaking, $p_\omega$ is the probability density function of $\psi_0(\zeta)$, $\zeta\sim p_\zeta$.
Note that the triggering kernel is jointly controlled by the deep network parameters and the weight matrix $ W $. We represent the Fourier feature generator as $G_\zeta$ and denote its parameters as  $\theta \in \Theta$. 

Figure~\ref{fig:fourier-kernel-exp} gives an illustrative example of representing the conditional intensity given sequence history using our approach. We choose $q = r = 2$ to visualize the noise prior $p_\zeta$ and the optimal spectra $p^*_\omega$ in a two-dimensional space. The optimal spectrum learned from data uniquely specifies a kernel function capable of capturing various non-linear triggering effects. Unlike Hawkes processes, underlying long-term influences of some events, in this case, can be preserved in the intensity function.

% \vspace{-0.1in}
\subsection{Efficient computation of log-likelihood function} 

As discussed in Section~\ref{sec:pp-deep-fourier-kernel}, the technical difficulty of evaluating the log-likelihood function is to perform the multi-dimensional integral of the kernel function. In particular, given a sequence of events $\boldsymbol{x}$, the log-likelihood function of our model can be written by substituting the conditional intensity function in \eqref{eq:pp-loglik} with \eqref{eq:lambda-fourier-kernel}, and thus we need to evaluate $\int_{\mathcal{X}} \lambda(x|\mathcal{H}_t;\theta) dx$. 
In many existing works, this term is carried out by numerical integration,
which can be computationally expensive. For instance, if we randomly sample $\kappa$ points in a $d$-dimensional space and the total number of events is $N$, the computational complexity will be $O(\kappa DN)$ ($\kappa \gg N^d$) using common numerical integration techniques.
Here we present a way to simplify the computation by deriving a closed-form expression for the integral as presented in the following proposition -- a benefit offered by the Fourier kernel.

%Given a sequence of events $\boldsymbol{x}$, the log-likelihood function of our model can be written by substituting the conditional intensity function in \eqref{eq:pp-loglik} with \eqref{eq:lambda-fourier-kernel}, to obtain
% \begin{equation}
%     \ell(\boldsymbol{x};\theta) = \sum_{i=1}^{N_T} \log \lambda(x_i|\mathcal{H}_{t_i};\theta) - \int_{\mathcal{X}} \lambda(x|\mathcal{H}_t;\theta) dx.
%     \label{eq:loglik-fourier-kernel}
% \end{equation}
% %
% A crucial step to tackle the computational challenge is to evaluate the integral in (\ref{eq:loglik-fourier-kernel}).
% %
% Normally, this term is carried out by some numerical integration techniques.

\begin{proposition}[Integral of conditional intensity function]
Let $t_{N_T+1} = T$ and $t_0 = 0$.
Given ordered events $\{x_1, \dots, x_{N_T}\}$ in the time horizon $[0, T]$. The integral term in the log-likelihood function %defined in %\eqref{eq:loglik-fourier-kernel} 
can be written as
%%%%%%%%%%%%%%%%%%
% For one column
%%%%%%%%%%%%%%%%%%
% \begin{equation}
% \begin{aligned}
%     & \int_{\mathcal{X}} \lambda(x|\mathcal{H}_t;\theta_0) dx 
%     = \mu T (b-a)^d + \frac{1}{D} \sum_{k=1}^D \sum_{i=1}^n \sum_{t_j < t_i} \\
%     & \cos\left(-\omega_k^\top W x_j \right)
%     \cos\left(\frac{t_{i+1} + t_{i}}{2}\right) \sin\left(\frac{t_{i+1} - t_{i}}{2}\right)
%     \cos^d\left(\frac{b + a}{2}\right) \sin^d\left(\frac{b-a}{2}\right)
%     \prod_{\ell=1}^d \frac{2 e^{\omega_k^\top w_\ell}}{\omega_k^\top w_\ell}.
%     \label{eq:loglik-integral}
% \end{aligned}
% \end{equation}
%%%%%%%%%%%%%%%%%%
% For double column
%%%%%%%%%%%%%%%%%%
\begin{equation}
\begin{aligned}
    & \int_{\mathcal{X}} \lambda(x|\mathcal{H}_t;\theta) dx 
    = \mu T (b-a)^d + \frac{1}{D} \sum_{k=1}^D \sum_{i=0}^{N_T} \sum_{t_j < t_i} \\
    & \cos\left(-\omega_k^\top W x_j \right)
    \cos\left(\frac{t_{i+1} + t_{i}}{2}\right) \sin\left(\frac{t_{i+1} - t_{i}}{2}\right) \cos^d\left(\frac{b + a}{2}\right) \sin^d\left(\frac{b-a}{2}\right)
    \prod_{\ell=1}^{d+1} \frac{2 e^{\omega_k^\top w_\ell}}{\omega_k^\top w_\ell},
    \label{eq:loglik-integral}
\end{aligned}
\end{equation}
where $w_\ell, \ell = 1,\dots, d$ is the $\ell$-th column vector in the matrix $W$, and $[a, b]$ are the range for each dimension of the mark space $\mathcal{M}$. 
Note that the computational complexity is $O(DN)$.
\label{prop:loglik-integral}
\end{proposition}
%The proof in given in Appendix~\ref{append:proof-prop-3}. 

\begin{remark}
From the right-hand side of \eqref{eq:loglik-integral}, the second term only depends on the weight matrix $W$, $D$ randomly sampled Fourier features, the time of events that occurred before $t$, and the region of the marked space. If we re-scale the range of each coordinate of the mark to be $[0 ,2\pi]$, i.e., $b=2\pi$ and $a=0$, then the second term of the integral equals to 0 and the integral defined in \eqref{eq:loglik-integral} can be further simplified as 
\[
    \int_{\mathcal{X}} \lambda(x|\mathcal{H}_t;\theta) dx = \mu T (2 \pi)^d.
\]
In particular, when we only consider time ($d=0$), the integral becomes:
% \begin{align*}
\[
    \int_{\mathcal{X}} \lambda(x|\mathcal{H}_t;\theta) dx = \mu T + \frac{1}{D} \sum_{k=1}^D \sum_{i=0}^{N_T} \sum_{t_j < t_i} \cos\left(-\omega_k^\top W x_j \right)
    \cos\left(\frac{t_{i+1} + t_{i}}{2}\right) \sin\left(\frac{t_{i+1} - t_{i}}{2}\right) \frac{2 e^{\omega_k^\top W}}{\omega_k^\top W}.
% \end{align*}
\]
\label{rmk:simple-integral}
\end{remark}

% \vspace{-.1in}
\subsection{Recursive computation of log-likelihood function}
\label{sec:online-detection}

Note that, leveraging the conditional probability decomposition, we can compute of the log-likelihood function $\ell(\boldsymbol{x}_{1:i};\theta^*)$ recursively:
\begin{equation}
\begin{aligned}
    \ell(\boldsymbol{x}_{1:1};\theta^*) =&~ \log f(x_1 | \mathcal{H}_{t_1});\\
    \ell(\boldsymbol{x}_{1:i};\theta^*) =&~\ell(\boldsymbol{x}_{1:i-1};\theta^*) + \log f(x_i | \mathcal{H}_{t_{i}}; \theta^*),~\forall i > 1,
    \label{eq:recursive-log-liklihood}
\end{aligned}
\end{equation}
where 
\[
    f(x_i | \mathcal{H}_{t_i};\theta) = \lambda(x_i|\mathcal{H}_{t_i};\theta) e^{ - \mu (t_i - t_{i-1}) (2\pi)^d}.
\]
This recursive expression makes it convenient to evaluate the detection statistic sequentially and perform online detection, which we summarize in Algorithm~\ref{algo:online-detection}.
%To perform the detection procedure in an online fashion, we develop a recursive formula to compute log-likelihood according to the definition of log-likelihood in Section~\ref{sec:background-pp}. Specifically,  The exact online detection algorithm is described in Algorithm~\ref{algo:online-detection}.

\begin{algorithm}[!t]
\SetAlgoLined
    {\bfseries Input:} An unknown sequence $\boldsymbol{x}$ with $N_T$ events and optimal model parameters $\theta^*, \varphi^*$\;
    Generate $D$ Fourier features from $G_\zeta(\theta^*)$ denoted as $\widehat{\Omega} = \{\omega_k\}_{k=1,\dots,D}$\;
    Generate $n'$ adversarial sequences from $G_z(\varphi^*)$ denoted as $\widehat{Z} = \{\boldsymbol{z}^l\}_{l=1,\dots,n'}$\;
    \While{$i \le N_T$}{
        Compute the log-likelihood $\ell(\boldsymbol{x}_{1:i};\theta^*)$ given $\widehat{\Omega}$ according to \eqref{eq:recursive-log-liklihood}\;
        $\eta^*_i \leftarrow 1/n^\prime \sum_{l=1}^{n^\prime} \ell(\boldsymbol{z}^l_{1:i}; \theta^*)$\;
        \If{$\ell(\boldsymbol{x}_{1:i};\theta) \ge \eta_i^*$}{
            Declare that it is an anomaly and record the stopping time $t_i$\;
        }
        $i \leftarrow i + 1$\;
    }
    Declare that it is not an anomaly\;
\caption{Online detection algorithm}
\label{algo:online-detection}
\end{algorithm}

%\vspace{-.1in}
\section{Adversarial sequence generator}
\label{sec:generator}

Now we describe the parameterization for the adversarial sequence generator. To achieve rich representation power for the adversarial generator $G_z$, we borrow the idea of the popular Recurrent Neural Network (RNN) structure. 
%
%However, RNN has limited capacity of representing the randomness 
%The policy class of the adversarial generator $G_z$ should be expressive enough to generate realistic counterfeit sequences. A typical way to do that is by using a Recurrent Neural Network (RNN) structure. However, the only source of randomness or variability in the RNN is found in the conditional output probability model. We suggest that this can be an inappropriate way to model the kind of variability observed in highly structured data, such as credit transactions, which is characterized by strong and complex dependencies among the output variables between different events. We argue that these complex dependencies cannot be modeled efficiently by the output probability models used in standard RNNs, which include either a simple unimodal distribution or a mixture of unimodal distributions \cite{Chung2015}.

\begin{figure}[!t]
\centering 
\includegraphics[width=.6\linewidth]{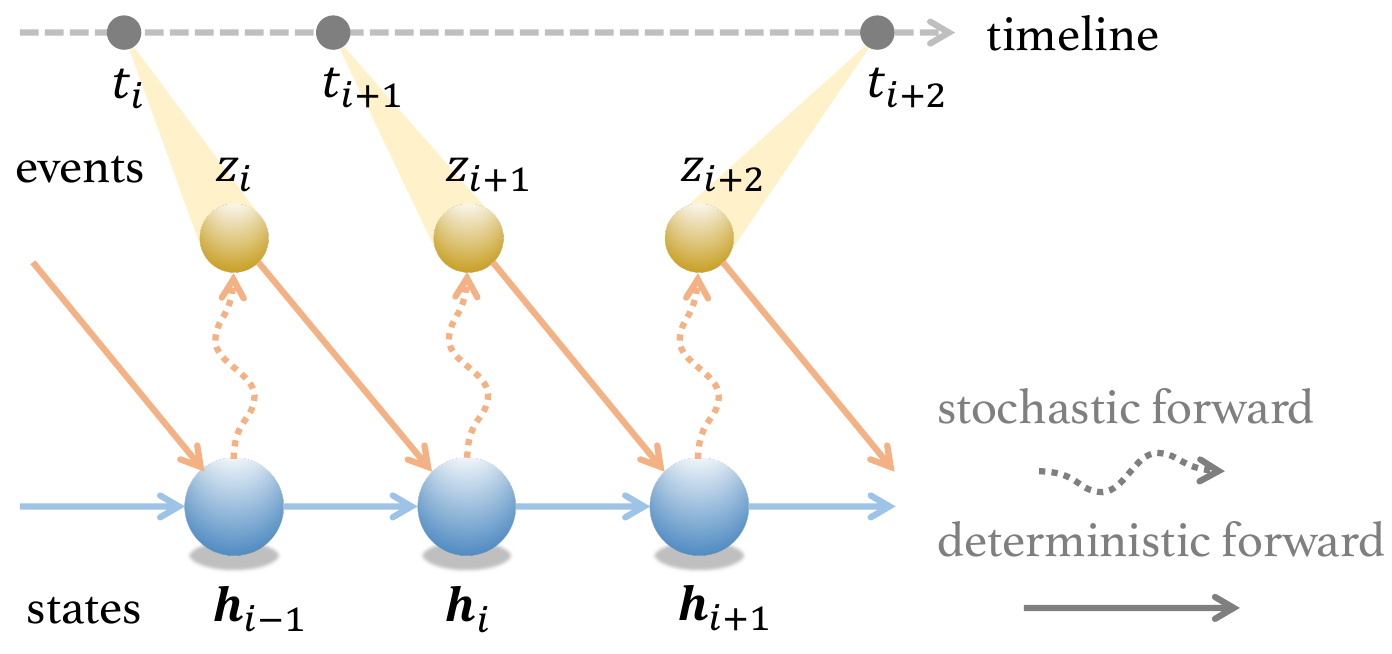}
\caption{RNN-based adversarial sequence generator.}
\label{fig:stoch-lstm}
\end{figure}

In particular, we develop an RNN-type generator with stochastic neurons \cite{Li2018, Chung2015} as shown in Figure~\ref{fig:stoch-lstm}, which can represent the nonlinear and long-range sequential dependency structure. 
Denote the $i$-th generated adversarial event as $z_i \coloneqq (t_{i-1} + \Delta t_i, m_i)$, where $\Delta t_i$ is the time interval between event $z_{i-1}, z_i$.
The generating process is described below:
\begin{align*}
    [\Delta t_i, m_i^\top]^\top & \sim \mathcal{N}(\boldsymbol{\mu}_{i-1}, \text{diag}(\boldsymbol{\sigma}_{i-1})),\\
    [\boldsymbol{\mu}_{i}, \boldsymbol{\sigma}_{i}^\top]^\top & = \psi_1(\boldsymbol{h}_i), \\
    \boldsymbol{h}_{i} & = \psi_2(\boldsymbol{h}_{i-1}, z_i),~~i=1,\dots,N_T,\\
    \boldsymbol{h}_0 & = \boldsymbol{0},~&
\end{align*}
where the hidden state $\boldsymbol{h}_i \in \mathbb{R}^{p}$ encodes the sequence of past events $\{z_1, \dots, z_{i-1}\},z_i\in\mathcal{X}$; $\mathcal{N}(\boldsymbol{\mu}, \Sigma)$ stands for the multivariate Gaussian distribution with mean $\boldsymbol{\mu} \in \mathbb{R}^{d+1}$ and covariance matrix $\Sigma \in \mathbb{R}^{({d+1})\times ({d+1})}$; here we only consider variance terms and thus the covariance matrix is diagonal with diagonal entries specified by a vector $\boldsymbol \sigma_i$, and $\mbox{diag}(\boldsymbol x)$ means to convert the vector $\boldsymbol x$ to a diagonal matrix.
Here we adopt the (two-sided) truncated normal distribution in our adversarial sequence generator by bounding the support of each mark to the interval $(a, b)$. 
The \emph{probabilistic density function (p.d.f.)} therefore is given by \emph{
$
\mathcal{N}(x | \mu, \sigma, a, b) = (1/\sigma)\phi((x-\mu)/\sigma) / ({\Phi((b-\mu)/\sigma) - \Phi((a-\mu)/\sigma})),
$}
where $\phi(x)$ is the p.d.f. of a standard normal distribution and $\Phi(x)$ is the corresponding \emph{cumulative density function (c.d.f.)}; $\mu, \sigma$ are represented by the LSTM structure, and $a, b$ are determined such that the percentage of density that lie within an interval for the normal is 99.7\% (the so-called  three-sigma rule-of-thumb).
Note that the process stops running until $t_{i} < T$ and $t_{i} + \Delta t_{i+1} \ge T$.

Function $\psi_2: \mathbb{R}^{p + d + 1} \rightarrow \mathbb{R}^{p}$ is an extended LSTM cell
and function $\psi_1: \mathbb{R}^p \rightarrow \mathbb{R}^{(d+1)^2 + d + 1}$ 
can be any nonlinear mappings.
There are two significant differences from the vanilla version of RNNs: (1) the outputs are sampled from hidden states rather than obtained by deterministic transformations (as in the vanilla version); randomly sampling will allow the learner to explore the events' space; (2) the sampled time point will be fed back to the RNN. 
We note that the model architecture for $\psi_1$ may be problem-specific. For example, $\psi_1$ can be represented by convolution neural network (CNN) \cite{lecun1995convolutional} if the high dimensional marks are images and can be represented by LSTM or \emph{Bidirectional Encoder Representations from Transformers (BERT)} \cite{DevlinCLT19} if the marks are text. In this paper, because the mark is three-dimensional, we use a fully-connected neural network to represent $\psi_1$, which achieves significantly better performance than baselines. The set of all trainable parameters in $\psi_1, \psi_2$ are denoted by $\varphi \in \mathcal{G}$.
%The proposed model aims to capture the two parts attributing to the hidden state. One is the deterministic influence from the previous hidden state, and the other is the stochastic influence from the latest sampled action.

% \subsection{Adversarial learning algorithm}
% \label{sec:learning}

\begin{algorithm}[!t]
\SetAlgoLined
    {\bfseries input:} dataset $X = \{\boldsymbol{x}^i\}_{i=1,\dots,n}$\;
    % where each sequence $\boldsymbol{x}^i = \{x_j^i\}_{j=1}^{N_i}$, $N_i$ is the number of events in the $i$-th sequence\;
    {\bfseries initialization:} model parameters $\theta, \varphi$\;
    \For{$1,\dots,M_0$}{
        (1) Randomly draw $n''$ training sequences from $X$ denoted as $\widehat{X} = \{\boldsymbol{x}^l \in X\}_{l=1,\dots,n''}$\;
        (2) Generate $n'$ adversarial sequences from $G_z(\varphi)$ denoted as $\widehat{Z} = \{\boldsymbol{z}^l\}_{l=1,\dots,n'}$\;
        (3) Generate $D$ Fourier features from $G_\zeta(\theta)$ denoted as $\widehat{\Omega} = \{\omega_k\}_{k=1,\dots,D}$\;
        Update $\varphi$ by descending gradient given $\widehat{X}, \widehat{Z}, \widehat{\Omega}$:
        \[
            \nabla_{\varphi} \frac{1}{n''}\sum_{l=1}^{n''} \ell(\boldsymbol{x}^l;\theta) - \frac{1}{n'}\sum_{l=1}^{n'} \ell(\boldsymbol{z}^l;\theta);
        \]\\
        \For{$1,\dots,M_1$}{
            Redo steps (1), (2), (3) to obtain new $\widehat{X}, \widehat{Z}, \widehat{\Omega}$\;
            Update $\theta$ by ascending gradient given $\widehat{X}, \widehat{Z}, \widehat{\Omega}$:
            \[
                \nabla_{\theta} \frac{1}{n''}\sum_{l=1}^{n''} \ell(\boldsymbol{x}^l;\theta) - \frac{1}{n'}\sum_{l=1}^{n'} \ell(\boldsymbol{z}^l;\theta);
            \]\\
        }
    }
\caption{Adversarial learning algorithm}
\label{algo:learning}
\end{algorithm}

We learn the adversarial detector's parameters in an off-line fashion by performing alternating minimization between optimizing the generator $G_z(\varphi)$ and optimizing the anomaly discriminator $\ell(\theta)$, using stochastic gradient descent. Let $M_0$ be the number of iterations, and $M_1$ be the number of steps to apply to the discriminator. 
Let $n', n'' < n$ be the number of generated adversarial sequences and the number of training sequences in a mini-batch, respectively. 
We follow the convention of choosing mini batch size in stochastic optimization algorithm \cite{li2014efficient}, and only require to use the same value for both $n'$ and $n''$. 
There is a clear trade-off between the model generalization and the estimation accuracy. 
Large $n'$ and $n''$ tend to converge to sharp minimizers of the training and testing functions, which lead to poorer generalization. 
In contrast, small $n'$ and $n''$ consistently converge to flat minimizers due to the inherent noise in the gradient estimation.
We also note that large $n'$ and $n''$ may cause the training to be computationally expensive. 
The learning process is summarized in Algorithm~\ref{algo:learning}.

\section{Numerical experiments}
\label{sec:experiments}

In this section, comprehensive numerical studies are presented to compare the proposed adversarial anomaly detector's performance with the state-of-the-art. 

\subsection{Comparison and performance metrics}
We compare our method (referred to as \texttt{AIL}) with \emph{four} state-of-the-art approaches: the one-class support vector machine \cite{Zhang2007} (\texttt{One-class SVM}), the cumulative sum of features extracted by principal component analysis \cite{Page1954} (\texttt{PCA+CUMCUM}), the local outlier factor \cite{Breunig2000} (\texttt{LOF}), and a recent work  leveraging IRL framework for sequential anomaly detection \cite{oh2019sequential} (\texttt{IRL-AD}).

The performance metrics are standard, including precision, recall, and $F_1$ score, all of which have been widely used in the information retrieval literature \cite{Michael2002}. This choice is because anomaly detection can be viewed as a binary classification problem, where the detector identifies if an unknown sequence is an anomaly. The $F_1$ score combines the \textit{precision} and \textit{recall}. 
Define the set of all true anomalous sequences as $U$, the set of positive sequences detected by the optimal detector as $V$.
Then precision $P$ and recall $R$ are defined by: 
\[
    P = |U \cap V|/|V|,~R = |U \cap V|/|U|,
\] 
where $|\cdot|$ is the number of elements in the set. 
The $F_1$ score is defined as $F_1 = 2 P R / (P + R)$ and the higher $F_1$ score the better. Since positive and negative samples in real data are highly unbalanced, we do not use the \emph{receiver operating characteristic (ROC)} curve (true positive rate versus false-positive rate) in our setting. 

\begin{figure*}[!t]
\centering
\begin{subfigure}[h]{0.99\linewidth}
\includegraphics[width=\linewidth]{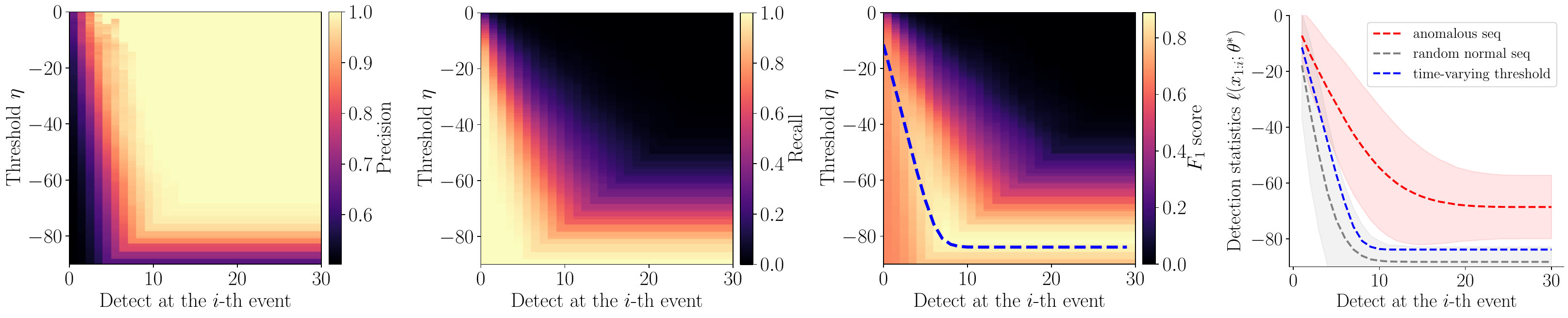}
\caption{singleton synthetic data}
\vspace{.1in}
\end{subfigure}
\vfill
\begin{subfigure}[h]{0.99\linewidth}
\includegraphics[width=\linewidth]{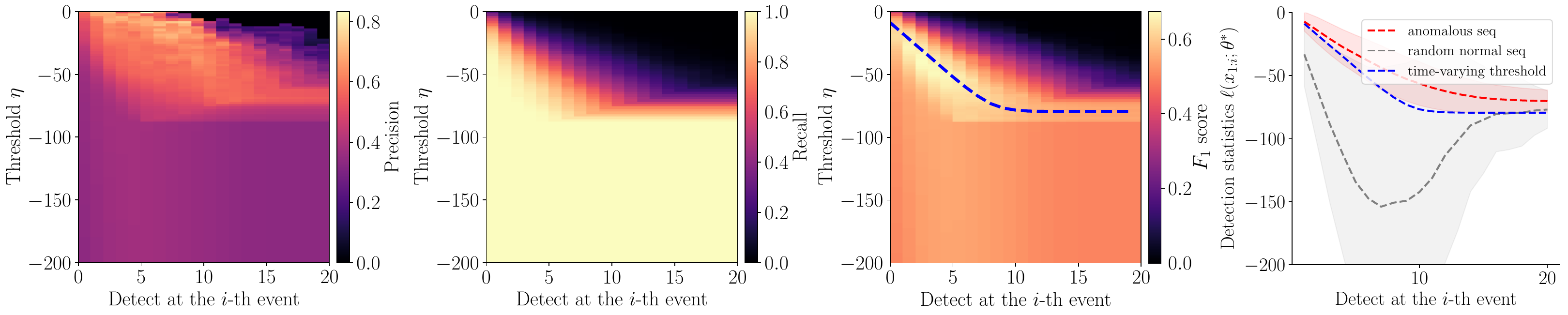}
\caption{composite synthetic data}
\vspace{.1in}
\end{subfigure}
\vfill
\begin{subfigure}[h]{0.99\linewidth}
\includegraphics[width=\linewidth]{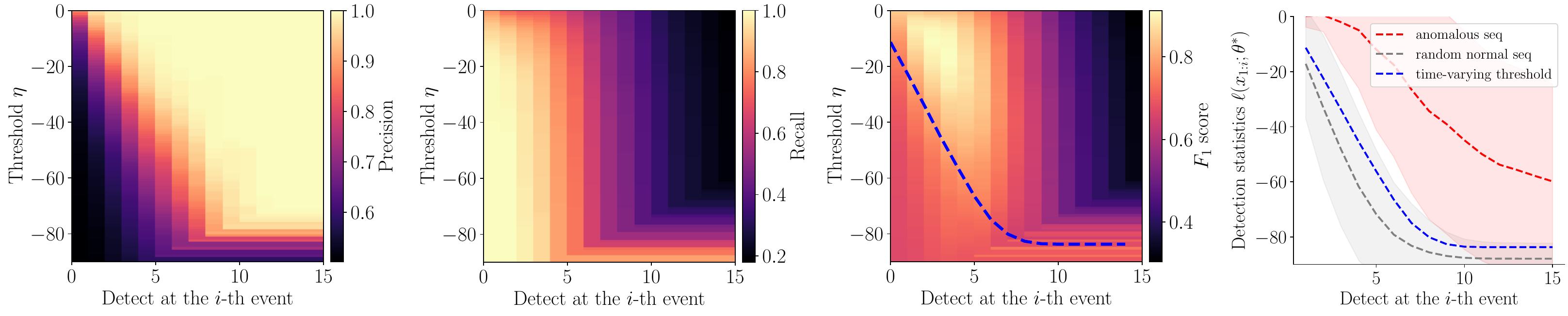}
\caption{real credit card fraud data}
\vspace{.1in}
\end{subfigure}
\vfill
\begin{subfigure}[h]{0.99\linewidth}
\includegraphics[width=\linewidth]{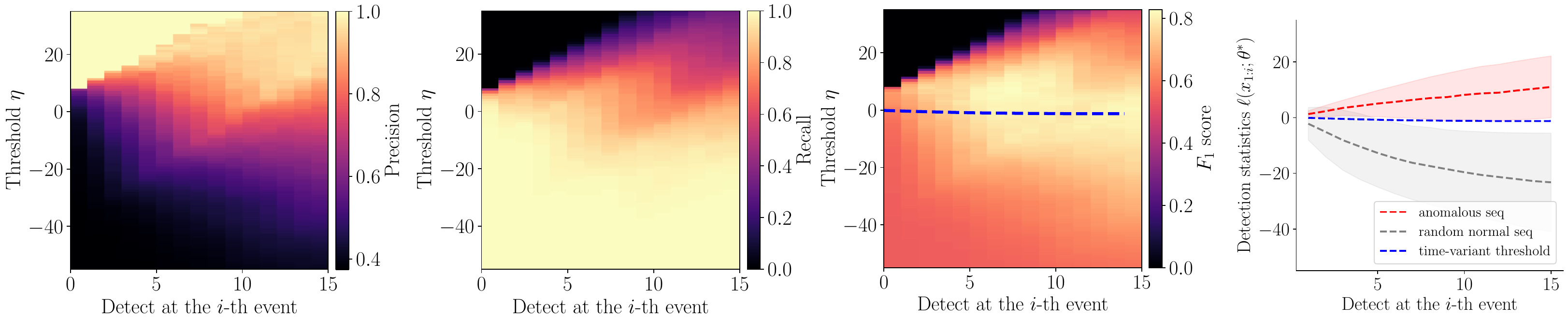}
\caption{real robbery data}
\vspace{.1in}
\end{subfigure}
\caption{The performance of our method (\texttt{AIL}) on four data sets. The first three columns correspond to the precision, recall, and $F_1$ score of our method using different thresholds. The \emph{dashed} lines in the third column indicate our time-varying thresholds. The fourth column shows the step-wise detection statistics for both anomalous and normal sequences.}
\label{fig:results}
\vspace{-.15in}
\end{figure*}

\begin{figure}[!t]
\centering
\begin{subfigure}[h]{0.24\linewidth}
\includegraphics[width=\linewidth]{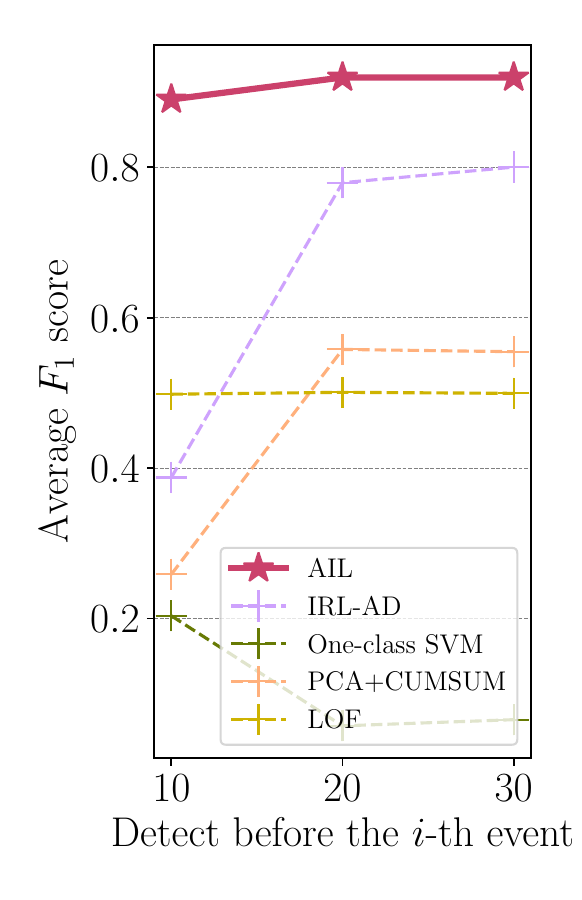}
\caption{singleton synthetic}
\end{subfigure}
\begin{subfigure}[h]{0.24\linewidth}
\includegraphics[width=\linewidth]{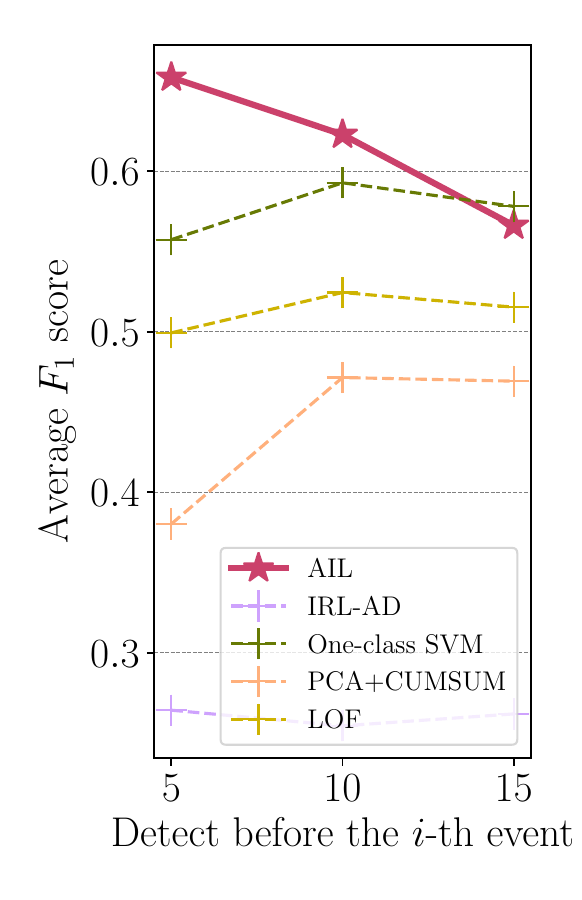}
\caption{composite synthetic}
\end{subfigure}
\begin{subfigure}[h]{0.24\linewidth}
\includegraphics[width=\linewidth]{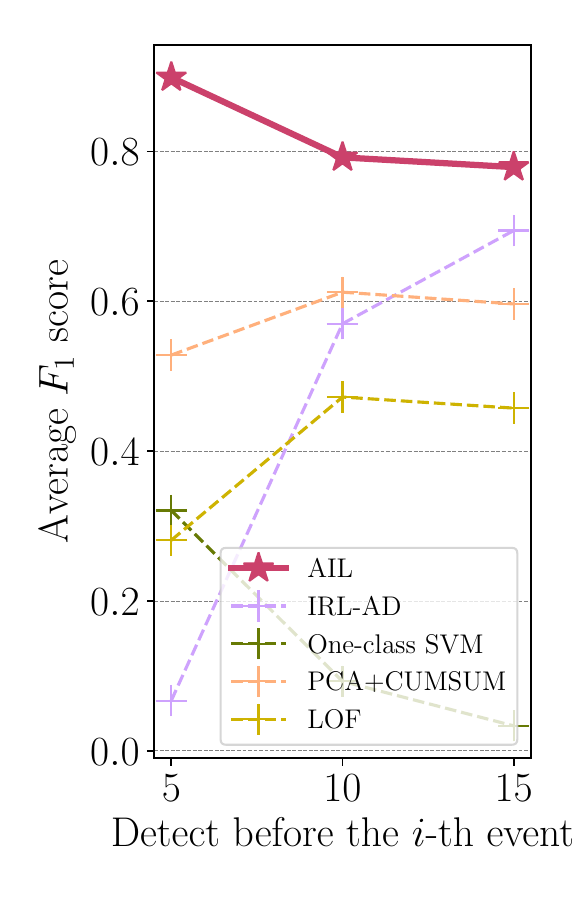}
\caption{credit card fraud}
\end{subfigure}
\begin{subfigure}[h]{0.24\linewidth}
\includegraphics[width=\linewidth]{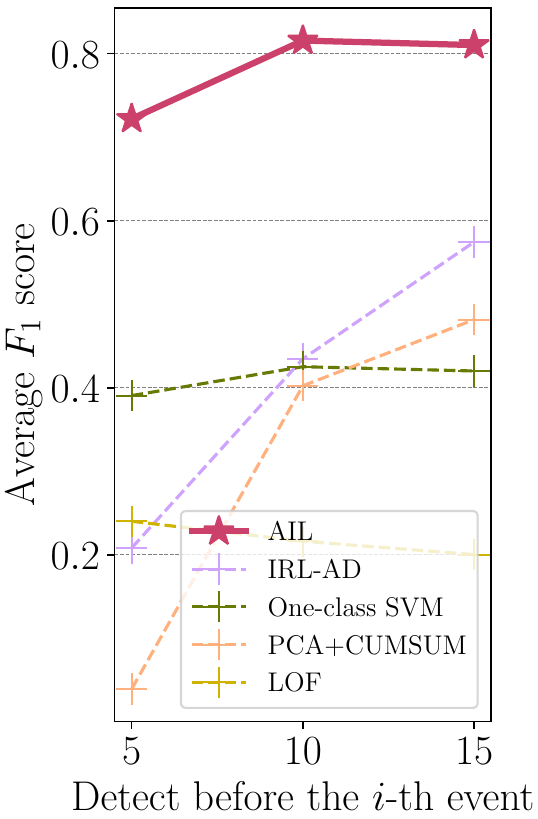}
\caption{robbery}
\end{subfigure}
\vspace{.1in}
\caption{The performance of our method (\texttt{AIL}) and four baselines on three data sets. The marks show the average $F_1$ score tested on testing sequences when decisions are made with observing part of the sequences.}
\label{fig:results-2}
\end{figure}

\subsection{Experiments set-up}
Consider two synthetic and \emph{two} real data sets:
(1) {\bf singleton synthetic data} consists of 1,000 anomalous sequences with an average length of 32. Each sequence is simulated by a Hawkes process with an exponential kernel specified in \eqref{eq:hawkes}, where $\beta=3$ and $\mu=10, \alpha=1$;
(2) {\bf composite synthetic data} consists of 1,000 mixed anomalous sequences with an average length of 29. Every 200 of the sequences are simulated by five Hawkes processes with different exponential kernels, where $\mu=10, \alpha=1$, and $\beta = 1,2,3,4,5$, respectively; and a real dataset
(3) {\bf real credit card fraud data} consists of 1,121 fraudulent credit transaction sequences with an average length of 21. Each anomalous transaction in a sequence includes the occurrence geolocation (latitude and longitude), time, and corresponding transaction amount in the dollar. 
(4) {\bf Robbery data} contains the 911-calls-for-service events in Atlanta from 2015 to 2017 (see, e.g., \cite{zhu2018crime, zhu2019crime, zhu2019spatial, zhu2021imitation}). We consider each crime series as a sequence of events: each event consists of the time (in seconds) and the geolocation (in latitude and longitude), indicating when and where the event occurred. 
We extract a series of events in the same category identified by the police detectives and treat them as one sequence. 
There exists intricate spatial and temporal dependency between these events with the same category. 
As indicated by \cite{zhu2019spatial}, the 911 calls of some crime incidents committed by the same individual share similar crime behaviours (e.g., forced entry) and tend to aggregate in time and space. 
This phenomenon is called modus operandi (M.O.) \cite{wang2015finding}. 
Within two years of data, this gives us 44 sequences with the sub-category of robbery. 
We test whether the algorithm can discriminate a series that is a robbery series or not. To create such an experiment, we also created 391 other types of crime series, which consist of randomly selected categories mixed together. We treat them as ``anomalous'' and ``normal'' data, respectively.
% , because the robbery, unlike other types of crimes, usually follows a particular modus operandi (M.O.), where criminal spots and times tend to have a triggering mechanism and can be treated as one category. 
In the experiments, we under-sample the Fourier features, where $D = 20$, to improve training efficiency. In addition, 
we select $n' = n'' = 32$ empirically based on the computational resource of the experimental set-up on a standard laptop with a quad-core 4.7 GHz processor.
The model obtains its convergence around $M_0 = 1,000$ iterations with $M_1 = 5$. 

Our evaluation procedure is described as follows. 
We consider two sets of simulation data and two sets of real data, respectively. 
Each data set is divided into 80\% for training and 20\% for testing. 
To evaluate the performance of the fitted model, we first mix the testing set with 5,000 normal sequences, which are simulated by multiple Poisson processes, and then perform online detection. Note that we do not simulate normal sequences for the robbery data experiment, since we treat other types of crime as the alternative. 
% We mix the above data sets, respectively, with 5,000 random ``normal'' sequences simulated by multiple Poisson processes. Each dataset is divided into 80\% training and 20\% testing data. 
% We perform the online detection procedure using each method when a new event occurs for all mixed testing sequences in each data set. 
The precision, recall, and $F_1$ score will be recorded accordingly. The method with higher precision, recall, and $F_1$ score at an earlier time step is more favorable than the others. 

\subsection{Results}

First, we summarize the performance of our method on three data sets in Figure~\ref{fig:results} and confirm that the proposed time-varying threshold can optimally separate the anomalies from normal sequences. 
To be specific, the fourth column in Figure~\ref{fig:results} shows the average log-likelihood (detection statistics) and its corresponding 1$\sigma$ region for both anomalous sequences and normal sequences. 
As we can see, the anomalous sequences attain a higher average log-likelihood than the normal sequences for all three data sets. Their log-likelihoods fall into different value ranges with rare overlap.  
Additionally, the time-varying threshold indicated by \emph{darker} dash lines lies between the value ranges of anomalous and normal sequences, which produces an amicable separation of these two types of sequences at any given time. The first three columns in Figure~\ref{fig:results} present more compelling evidence that the time-varying threshold is near-optimal. Colored cells of these heat-maps are calculated with different constant thresholds $\eta$ at each step $i$ by performing cross-validation. The brightest regions indicate the ``ground truth'' of the optimal choices of the threshold. 
As shown in the third column, the time-varying thresholds are very close to the optimal choices found by cross-validation. 

\begin{table}[]
\centering
\caption{$F_1$ score before $i$-th event using different adversarial generators in the proposed framework.}
%\vspace{.1in}
\label{tab:generator-comparison}
\resizebox{.8\linewidth}{!}{%
\begin{tabular}{lccccccc}
\hline
%\up
\multirow{2}{*}{Generator in AIL} & \multicolumn{3}{c}{Singleton synthetic data}    &  & \multicolumn{3}{c}{Composite synthetic data}    \\
                                  & $i=5$ & $i=10$ & $i=15$ &  & $i=5$ & $i=10$ & $i=15$ \\ \hline
%                                  \up
\emph{vanilla Hawkes process}            & .821          & .889           & .911           &  & .421          & .411           & .370           \\
\emph{vanilla LSTM}                      & .761          & .830           & .878           &  & .594          & .542           & .519            \\
\emph{proposed extended LSTM}            & .888          & .916           & .916           &  & .658          & .623           & .566           \\ \hline
\end{tabular}%
}
\end{table}

We also compare the step-wise $F_1$ scores of our method with the other four baselines in Figure~\ref{fig:results-2}. 
The results show that (1) from an overall standpoint, our method outperforms other baselines with significantly higher $F_1$ scores and (2) our method allows for easier and faster detection of anomalous sequences (before ten events being observed in our experiments), which is critically vital in sequential scenarios for most of the applications.

% In addition, using a point process model as the generator would be limited to represent the dynamics of the adversarial sequences, which deteriorates the performance of the detection in practice. 
% Here we add an ablation study to investigate the performance of our method on the real robbery data using two different generators, as shown in Table~\ref{tab:generator-comparison}. 
Finally, we present an ablation study to investigate the performance of our method using different generators. As shown in Table~\ref{tab:generator-comparison}, the proposed generator based on an extended LSTM structure significantly outperforms other generators in step-wise $F_1$ score. 
As a sanity check, the generator using the vanilla Hawkes process achieves competitive performances on the singleton synthetic data since the true anomalous sequences are from a Hawkes process; 
However, we can observe a dramatic performance deterioration on the composite synthetic data. The anomalous sequences are generated by multiple distributions and can hardly be captured by the vanilla Hawkes process. This result confirms that using a generic generative model cannot achieve the best performance. 

\section{Conclusion and Discussions}

We have presented a novel unsupervised anomaly detection framework on sequential data based on adversarial learning. 
A robust detector can be found by solving a minimax problem, and the optimal generator also helps define the time-varying threshold for making decisions in an online fashion. 
We model the sequential event data using a marked point process model with a neural Fourier kernel. 
Using both synthetic and real data, we demonstrated that our proposed approach outperforms other state-of-the-art. 
In particular, the experimental results suggest that the proposed framework has achieved excellent performance on a proprietary large-scale credit-card fraud dataset from a major department store in the U.S., which shows the potential of proposed methods to apply to real-world problems. 

Given the prevalence of sequential event data (in many applications, there is only one-class data), we believe our proposed method can be broadly applicable to many scenarios. Such applications include financial anomaly detection, internet intrusion detection, system anomaly detection such as power systems cascading failures, all of which are sequential discrete events data with complex temporal dependence. 
On the methodology side, we believe the proposed framework is a natural way to tackle the one-class anomaly detection problem, leveraging adversarial learning advances. It may provide a first step towards bridging imitation learning and sequential anomaly detection.

\bibliography{draft}
\bibliographystyle{plain}
 
 \newpage
 \appendix

\section{Deriving Conditional Intensity of MTPPs}
\label{append:conditional-intensity}
 
% The conditional intensity function is defined by:
% \begin{align*}
%     \lambda(t,m|\mathcal{H}_t) 
%     = \frac{f(t,m|\mathcal{H}_t)}{1-F(t,m|\mathcal{H}_t)} 
%     = \frac{f(t,m|\mathcal{H}_{t_n})}{1-F(t,m|\mathcal{H}_{t_n})}.
% \end{align*}
% The last equality is achieved by consider up to time $t$, the filtration $\mathcal{H}_t$ contains $n$ history events $x_1, x_2,\cdots, x_n$, thus $\mathcal{H}_t$ is equivalent to be written as $\mathcal{H}_{t_n}$.

% Then $\lambda(t,m|\mathcal{H}_t)$ can be interpreted heuristically as the following: consider a small enough interval around $t$ addressed as $ dt$, and a ball $B(m, dm)$ with radius $dm$, then
% \[
%     \lambda(t,m|\mathcal{H}_{t_n}) dt dm =\frac{f(t,m|\mathcal{H}_{t_n})dt dm}{1-F(t,m|\mathcal{H}_{t_n})}.
% \]
% Denote $\Omega = [t,t+dt]\times B(m,dm)$, the above equation can be written as
% \begin{align*}
%     \frac{\mathbb{P}({x_{n+1}\in \Omega}|\mathcal{H}_{t_n})}{\mathbb{P}(t_{n+1}\notin (t_n,t)|\mathcal{H}_{t_n})}
%     =&~\frac{\mathbb{P}({x_{n+1}\in \Omega},t_{n+1}\notin (t_n,t)|\mathcal{H}_{t_n})}{\mathbb{P}(t_{n+1}\notin (t_n,t)|\mathcal{H}_{t_n})}
%     =\mathbb{P}({x_{n+1}\in \Omega}|t_{n+1}\notin (t_n,t),\mathcal{H}_{t_n})\\
%     =&~\mathbb{P}({x_{n+1}\in \Omega}|\mathcal{H}_{t})
%     =\mathbb{E}[N(\Omega)|\mathcal{H}_t].
% \end{align*}
% This shows the definition we present in main section \ref{sec:background-pp} is equivalent. %to \eqref{eq:def-conditional-intensity}.

Assume that we have total number of $N([0, T] \times \mathcal{M})$ observations in $\boldsymbol{x}$. For any given $t \in [0, T]$, we assume that $n$ events happened before $t$ and denote the occurrence time of the latest event as $t_n$. Let $\Omega = [t, t + dt) \times B(m, dm)$ where $m \in \mathcal{M}$. Let $F(t) = \mathbb{P}(x_{n+1}, t_{n+1} < t | \mathcal{H}_{t_n} \cup x_n)$ be the conditional cumulative probability function, and $\mathcal{H}_{t_n} \cup x_n$ represents the history events happened up to time $t_n$ and at $t_n$. Let $f(t, m) \triangleq f(t, m|\mathcal{H}_{t_n} \cup x_n)$ be the corresponding conditional probability density function of new event happening in $\Omega$. As defined in \eqref{eq:def-conditional-intensity}, $\lambda(t, m)$ can be expressed as
\[
    \begin{aligned}
        \lambda(t, m) &= \mathbb{P}\{x_{n+1} \in \Omega | \mathcal{H}_{t}\} = \mathbb{P}\{x_{n+1} \in \Omega | \mathcal{H}_{t_n} \cup x_n \cup \{t_{n+1} \geq t\}\} \\&= \frac{\mathbb{P}\{x_{n+1} \in \Omega, t_{n+1} \geq t | \mathcal{H}_{t_{n}} \cup x_{n}\}}{\mathbb{P}\{t_{n+1} \geq t | \mathcal{H}_{t_{n}} \cup x_{n}\}} \\&= \frac{f(t, m)}{1 - F(t)}.
    \end{aligned} 
\]
We multiply the differential of time and space $dtdm$ on both side of the equation, and integral over $m$
\[
    dt \cdot \int_{\mathcal{M}}\lambda(t, u)du = \frac{dt \cdot \int_{\mathcal{M}}f(t, u)du}{1 - F(t)} = \frac{dF(t)}{1 - F(t)} = -d\log{(1 - F(t))}.
\]
Hence, integrating over $t$ on $(t_n, t)$ leads to $F(t) = 1 - \exp (-\int_{t_{n}}^{t}\int_{\mathcal{M}} \lambda(\tau, u)dud\tau)$ because $F(t_{n}) = 0$. Then we have
\begin{equation}
    f(t, m) = \lambda(t, m) \cdot \exp \left( -\int_{t_{n}}^{t}\int_{\mathcal{M}} \lambda(\tau, u)dud\tau \right)\ ,
    \nonumber
\end{equation}
The joint p.d.f. for a realization is then, by the chain rule, $f(x_1, \dots, x_{N([0, T]\times \mathcal{M})}) = \prod_{i=1}^{N([0, T]\times \mathcal{M})} f(t_i, m_i)$. Then the log-likelihood of an observed sequence $\boldsymbol{x}$ can be written as
\[
    l(\boldsymbol{x}) = \sum_{i=1}^{N([0, T]\times \mathcal{M})} \log \lambda(t_i, m_i) - \int_{0}^{T}\int_{\mathcal{M}} \lambda(\tau, u)dud\tau .
\]

\section{Deriving Log-Likelihood of MTPPs}
\label{append:Likelihood}
The likelihood function is defined as:
\begin{align*}
     \mathcal{L} 
     =f(t_1,\dots,t_n;m_1,\dots,m_n)
     =f(t_1,m_1|\mathcal{H}_{t_1})\times\dots\times f(t_n,m_n|\mathcal{H}_{t_n})
     =\prod_{i=1}^n f(t_i,m_i|\mathcal{H}_{t_i})
\end{align*}
Note from the definition of conditional intensity function, we get:
\begin{align*}
    \int_\mathcal{M} \lambda(t,m|\mathcal{H}_t) 
    = &~\frac{\int_\mathcal{M} f(t,m|\mathcal{H}_{t})}{1-F(t|\mathcal{H}_{t})} 
    % = \frac{\partial^2}{\partial t \partial m}\frac{F(t,m|\mathcal{H}_{t_n})}{1-F(t,m|\mathcal{H}_{t_n})} 
    = -\frac{\partial}{\partial t} \log (1-F(t|\mathcal{H}_{t_n})).
\end{align*}
Integrating both side from $t_n$ to $T$ (since $\lambda(t,m|\mathcal{H}_t)$ depends on history events, so its support is $[t_n,T)$) where $t_n$ is the last event before $T$. Integrate over all $\mathcal{M}$, we can get:
$
    \int_{t_n}^T \int_{m\in \mathcal{M}} \lambda(t,m|\mathcal{H}_t)dtdm -\log \left(1-F(t,m|\mathcal{H}_{t_n})\right)
$,
obviously, using basic calculus we could find:
% \begin{align*}
\[
    f(t,m|\mathcal{H}_{t}) = \lambda(t,m|\mathcal{H}_{t}) \cdot \exp\bigg( - \int_{t_n}^T \int_{m\in \mathcal{M}} \lambda(t,m|\mathcal{H}_{t})dtdm \bigg).
\]
% \end{align*}
Plugging in the above formula into the definition of likelihood function, we have:
% \begin{equation}
% \begin{aligned}
\[
    \mathcal{L} = \prod_{i=1}^n\bigg(\lambda(t_i,m_i|\mathcal{H}_{t})\bigg) \cdot \exp\bigg(-\int_{0}^T\int_{m\in \mathcal{M}}\lambda(t,m|\mathcal{H}_{t})dtdm\bigg),
\]
% \end{aligned}
% \end{equation}
and the log-likelihood function of marked spatio-temporal point process can be the written as :
% \begin{equation}
\[
    \ell = \sum_{i=1}^{n} \log \lambda(t_i, m_i|\mathcal{H}_{t}) \\
    - \int_0^T \int_{m \in \mathcal{M}} \lambda(t, m|\mathcal{H}_{t}) dt dm.
\]
% \end{equation}

\section{Proof for Proposition~\ref{prop:triggering-func}}
\label{append:proof-prop-1}

For the notational simplicity, we denote $W x$ as $x$.
First, since both $K$ and $p_\omega$ are real-valued, it suffices to consider only the real portion of $e^{ix}$ when invoking Theorem~\ref{thm1}. Thus, using $\text{Re}[e^{ix}] = \text{Re}[\cos(x) + i \sin(x)] = cos(x)$, we have
\[
    K(x, x^\prime) = \text{Re}[K(x, x^\prime)] = \int_\Omega p_\omega(\omega) \cos(\omega^\top (x - x^\prime)) d\omega.
\]
Next, we have
\begingroup
\allowdisplaybreaks
\begin{align*}
    \int_\Omega p_\omega(\omega) \cos\left(\omega^\top (x - x^\prime)\right) d\omega
    \overset{(i)}{=} &~ \int_\Omega p_\omega(\omega) \cos\left(\omega^\top (x - x^\prime)\right) d\omega \\
    + &~ \int_\Omega \int_0^{2\pi} \frac{1}{2\pi} p_\omega(\omega) \cos\left( \omega^\top (x+x^\prime) + 2 u \right) du d\omega \\
    = &~ \int_\Omega \int_0^{2\pi} \frac{1}{2\pi} p_\omega(\omega) \big[ \cos\left(\omega^\top (x - x^\prime)\right) + \cos\left( \omega^\top (x+x^\prime) + 2 u \right) \big] du d\omega\\
    = &~ \int_\Omega \int_0^{2\pi} \frac{1}{2\pi} p_\omega(\omega) \left[ 2 \cos(\omega^\top x + u) \cdot \cos(\omega^\top x^\prime + u) \right] du d\omega\\
    = &~ \int_\Omega p_\omega(\omega) \int_0^{2\pi} \frac{1}{2\pi} \bigg[ \sqrt{2} \cos(\omega^\top x + u) \cdot \sqrt{2} \cos(\omega^\top x^\prime + u) \bigg] du d\omega\\
    = &~ \mathbb{E} \left [ \phi_\omega(x) \cdot \phi_\omega(x^\prime) \right ]. 
\end{align*}
\endgroup
where $\phi_\omega(x) \coloneqq \sqrt{2} \cos(\omega^\top x + u)$, $\omega$ is sampled from $p_\omega$, and $u$ is uniformly sampled from $[0, 2\pi]$. The equation $(i)$ holds since the second term equals to 0 as shown below:
\begin{align*}
    \int_\Omega \int_0^{2\pi} p_\omega(\omega) \cos\left( \omega^\top (x+x^\prime) + 2 u \right) du d\omega
    =&~\int_\Omega p_\omega(\omega) \int_0^{2\pi} \cos\left( \omega^\top (x+x^\prime) + 2 u \right) du d\omega \\
    =&~\int_\Omega p_\omega(\omega) \cdot 0 \cdot d\omega
    = 0.
\end{align*}
Therefore, we can obtain the result in Proposition~\ref{prop:triggering-func}.

\section{Proof for Proposition~\ref{prop:triggering-func-convergence}}
\label{append:proof-prop-2}

Similar to the proof in Appendix~\ref{append:proof-prop-1}, we denote $W x$ as $x \in \mathcal{X}$ for the notational simplicity. Recall that we denote $R$ as the radius of the Euclidean ball containing $\mathcal{X}$ in Section~\ref{sec:pp-deep-fourier-kernel}. 
In the following, we first present two useful lemmas.

\begin{lemma}
\label{lemma:prop-2-lemma-1}
Assume $\mathcal{X} \subset \mathbb{R}^d$ is compact. Let $R$ denote the radius of the Euclidean ball containing $\mathcal{X}$, then for the kernel-induced feature mapping $\Phi$ defined in \eqref{eq:triggering-func-estimation}, the following holds for any $0 < r \le 2R$ and $\epsilon > 0$:
% \begin{align*}
\[
    \mathbb{P}\left\{\underset{x, x^\prime \in \mathcal{X}}{\sup} \left | \Phi(x)^\top \Phi(x^\prime) - K(x, x^\prime) \right | \ge \epsilon \right\} \le ~2 \mathcal{N}(2R, r) \exp\left\{ - \frac{D \epsilon^2}{8}\right\} + \frac{4r\sigma_p}{\epsilon}.
\]
% \end{align*}
where $\sigma_p^2 = \mathbb{E}_{\omega \sim p_\omega} [\omega^\top \omega] < \infty$ is the second moment of the Fourier features, and $\mathcal{N}(R, r)$ denotes the minimal number of balls of radius $r$ needed to cover a ball of radius $R$.
\end{lemma}

\begin{proof}{Proof of Lemma \ref{lemma:prop-2-lemma-1}}
Now, define $\Delta = \{\delta: \delta = x - x^\prime,~, x,x^\prime \in \mathcal{X} \}$ and note that $\Delta$ is contained in a ball of radius at most $2R$. $\Delta$ is a closed set since $\mathcal{X}$ is closed and thus $\Delta$ is a compact set. 
Define $B = \mathcal{N}(2R, r)$ the number of balls of radius $r$ needed to cover $\Delta$ and let $\delta_j$, for $j \in [B]$ denote the center of the covering balls. 
Thus, for any $\delta \in \Delta$ there exists a $j$ such that $\delta = \delta_j + r^\prime$ where $|r^\prime| < r$.

Next, we define $S(\delta) = \Phi(x)^\top \Phi(x^\top) - K(x, x^\prime)$, where $\delta = x - x^\prime$. Since $S$ is continuously differentiable over the compact set $\Delta$, it is $L$-Lipschitz with $L = \sup_{\delta \in \Delta} || \nabla S(\delta) ||$. Note that if we assume $L < \epsilon/2r$ and for all $j \in [B]$ we have $|S(\delta_j)| < \epsilon/2$, then the following inequality holds for all $\delta = \delta_j + r^\prime \in \Delta$:
\begin{equation}
    |S(\delta)| = |S(\delta_j + r^\prime)| \le L |\delta_j - (\delta_j + r^\prime)| + |S(\delta_j)| \le rL + \frac{\epsilon}{2} < \epsilon.
    \label{eq:prop-2-part-1}
\end{equation}
The remainder of this proof bounds the probability of the events $L > \epsilon / (2r)$ and $|S(\delta_j)|\ge \epsilon/2$. Note that all following probabilities and expectations are with respect to the random variables $\omega_1, \dots, \omega_D$.

To bound the probability of the first event, we use Proposition~\ref{prop:triggering-func} and the linearity of expectation, which implies the key fact $\mathbb{E}[\nabla(\Phi(x)^\top \Phi(x^\prime))] = \nabla K(x, x^\top)$. We proceed with the following series of inequalities:
\begingroup
\allowdisplaybreaks
\begin{align*}
    \mathbb{E}\left[L^2\right] 
    = &~\mathbb{E}\left[\underset{\delta\in\Delta}{\sup}||\nabla S(\delta)||^2\right]\\
    = &~\mathbb{E}\bigg[\underset{x, x^\prime \in \mathcal{X}}{\sup}||\nabla(\Phi(x)^\top \Phi(x^\prime)) - \nabla K(x, x^\prime) ||^2\bigg]\\
    \overset{(i)}{\le}&~2 \mathbb{E}\left[\underset{x, x^\prime \in \mathcal{X}}{\sup}||\nabla(\Phi(x)^\top \Phi(x^\prime))||^2\right] + 2 \underset{x, x^\prime \in \mathcal{X}}{\sup}||\nabla K(x, x^\prime) ||^2\\
    = &~2 \mathbb{E}\left[\underset{x, x^\prime \in \mathcal{X}}{\sup}||\nabla(\Phi(x)^\top \Phi(x^\prime))||^2\right] + 2 \underset{x, x^\prime \in \mathcal{X}}{\sup}||\mathbb{E}\left[ \nabla(\Phi(x)^\top \Phi(x^\prime))\right] ||^2\\
    \overset{(ii)}{\le} &~4 \mathbb{E}\left[\underset{x, x^\prime \in \mathcal{X}}{\sup}||\nabla(\Phi(x)^\top \Phi(x^\prime))||^2\right],
\end{align*}
\endgroup
where the first inequality $(i)$ holds due to the the inequality $||a + b||^2 \le 2 ||a||^2 + 2 ||b||^2$ (which follows from Jensen’s inequality) and the subadditivity of the supremum function. The second inequality $(ii)$ also holds by Jensen’s inequality (applied twice) and again the subadditivity of supremum function. Furthermore, using a sum-difference trigonometric identity and computing the gradient with respect to $\delta = x - x^\prime$, yield the following for any $x, x^\prime \in \mathcal{X}$:
\[
% \begin{align*}
    \nabla(\Phi(x)^\top \Phi(x^\prime)) = \nabla \left(\frac{1}{D} \sum_{k=1}^D \cos(\omega_k^\top (x-x^\prime))\right) = \frac{1}{D} \sum_{k=1}^D \omega_k \sin(\omega_k^\top (x-x^\prime)).
% \end{align*}
\]
Combining the two previous results gives
\begin{align*}
    \mathbb{E}[L^2] 
    & \le 4 \mathbb{E}\left[ \underset{x, x^\prime \in \mathcal{X}}{\sup} ||\frac{1}{D} \sum_{k=1}^D \omega_k \sin(\omega_k^\top (x-x^\prime)) ||^2 \right]\\
    & \le 4 \underset{\omega_1, \dots, \omega_D}{\mathbb{E}} \left[ \left( \frac{1}{D}\sum_{k=1}^{D} ||\omega_k|| \right)^2 \right]\\
    & \le 4 \underset{\omega_1, \dots, \omega_D}{\mathbb{E}} \left[ \frac{1}{D}\sum_{k=1}^{D} ||\omega_k||^2 \right] = 4 \underset{\omega}{\mathbb{E}}[||\omega||^2] = 4 \sigma_p^2,
\end{align*}
which follows from the triangle inequality, $|\sin(\cdot)|\le 1$, Jensen’s inequality and the fact that the $\omega_k$s are drawn i.i.d. derive the final expression. Thus, we can bound the probability of the first event via Markov’s inequality:
\begin{equation}
    \mathbb{P}\left[ L \ge \frac{\epsilon}{2r} \right] \le \left( \frac{4r\sigma_p}{\epsilon} \right)^2.
    \label{eq:prop-2-part-2}
\end{equation}
To bound the probability of the second event, note that, by definition, $S(\delta)$ is a sum of $D$ i.i.d. variables, each bounded in absolute value by $\frac{2}{D}$ (since, for all $x$ and $x^\prime$, we have $|K(x, x^\prime)| \le 1$ and $|\Phi(x)^\top \Phi(x^\prime)| \le 1$), and $\mathbb{E}[S(\delta)] = 0$. Thus, by Hoeffding’s inequality and the union bound, we can write
\begin{equation}
% \begin{aligned}
    \mathbb{P}\left[\exists j \in [B]: |S(\delta_j)|\ge \frac{\epsilon}{2}\right] \le ~\sum_{j=1}^{B} \mathbb{P}\left[|S(\delta_j)|\ge \frac{\epsilon}{2}\right] \le ~2 B \exp\left( -\frac{D\epsilon^2}{8} \right).
    \label{eq:prop-2-part-3}
% \end{aligned}
\end{equation}
Finally, combining \eqref{eq:prop-2-part-1}, \eqref{eq:prop-2-part-2}, \eqref{eq:prop-2-part-3}, and the definition of $B$ we have the result in Proposition~\ref{prop:triggering-func-convergence}, i.e.,
\[
    \mathbb{P}\left[\underset{\delta \in \Delta}{\sup}|S(\delta_j)|\ge \epsilon\right] \le 2 \mathcal{N}(2 R, r) \exp\left\{ - \frac{D \epsilon^2}{8} \right\} + \left(\frac{4r\sigma_p}{\epsilon}\right)^2,
\]
\end{proof}

As we can see now, a key factor in the bound of the proposition is the covering number $N(2R,r)$, which strongly depends on the dimension of the space $d$. In the following proof, we make this dependency explicit for one especially simple case, although similar arguments hold for more general scenarios as well.

\begin{lemma}
\label{lemma:prop-2-lemma-2}
Let $\mathcal{X} \subset \mathbb{R}^d$ be a compact and let $R$ denote the radius of the smallest enclosing ball. Then, the following inequality holds:
\[
    \mathcal{N}(R, r) \le \left( \frac{3R}{r} \right)^d.
\]
\end{lemma}
\begin{proof}{Proof of Lemma \ref{lemma:prop-2-lemma-2}}
By using the volume of balls in $\mathbb R^d$, we already see that $R^d / (r/3)^d = (3R/r)^d$ is a trivial upper bound on the number of balls of radius $r/3$ that can be packed into a ball of radius $R$ without intersecting. Now, consider a maximal packing of at most $(3R/r)^d$ balls of radius $r/3$ into the ball of radius $R$. Every point in the ball of radius $R$ is at distance at most $r$ from the center of at least one of the packing balls. If this were not true, we would be able to fit another ball into the packing, thereby contradicting the assumption that it is a maximal packing. Thus, if we grow the radius of the at most $(3R/r)^d$ balls to $r$, they will then provide a (not necessarily minimal) cover of the ball of radius $R$.
\end{proof}

Finally, by combining the two previous lemmas, we can present an explicit finite sample approximation bound. We use lemma~\ref{lemma:prop-2-lemma-1} in conjunction with lemma~\ref{lemma:prop-2-lemma-2} with the following choice of $r$:
\[
    r = \left[ \frac{2(6R)^d \exp(- \frac{D\epsilon^2}{8})}{\left( \frac{4\sigma_p}{\epsilon} \right)^2} \right]^{\frac{2}{d+2}},
\]
which results in the following expression
\[
    \mathbb{P}\left[ \underset{\delta \in \Delta}{\sup} |S(\delta)| \ge \epsilon \right] \le 4 \left( \frac{24R\sigma_p}{\epsilon} \right)^{\frac{2d}{d+2}} \exp\left( -\frac{D\epsilon^2}{4(d+2)} \right).
\]
Since $32R \sigma_p/\epsilon \ge 1$, the exponent $2d / (d+2)$ can be replaced by 2, which completes the proof.

\section{Proof for Proposition~\ref{prop:loglik-integral}}
\label{append:proof-prop-3}

To calculate the integral (the second term of the log-likelihood function defined in \eqref{eq:pp-loglik}), 
we first need consider the time and the mark of events separately, i.e., $x = [t, m]^\top$. 
Denote $\boldsymbol{i}$ as the imaginary unit. 
Hence the integral can be written as
\begingroup
\allowdisplaybreaks
% \begin{equation}
\begin{align}
% \begin{split}
    &~\int_0^T \int_\mathcal{M} \left(\mu + \sum_{t_j < t} \widetilde K([t, m]^\top, [t_j, m_j]^\top) \right) dm dt \nonumber\\
    = &~\int_0^T \int_\mathcal{M} \left(\mu + 
    \sum_{t_j < t} \frac{1}{D} \sum_{k=1}^D 
    e^{\boldsymbol{i}\omega_k^\top W ([t, m]^\top - [t_j, m_j]^\top)} 
    \right) dm dt \nonumber\\
    = &~\mu T |\mathcal{M}|+ 
    \frac{1}{D} \sum_{k=1}^D \int_0^T \int_\mathcal{M} \sum_{t_j < t} 
    e^{\boldsymbol{i}\omega_k^\top W ([t, m]^\top - [t_j, m_j]^\top)}
    dm dt \nonumber\\
    = &~\mu T |\mathcal{M}| + \frac{1}{D} \sum_{k=1}^D \sum_{i=0}^{N_T} \int^{t_{i+1}}_{t_{i}} \int_\mathcal{M} \sum_{t_j < t}e^{\boldsymbol{i}\omega_k^\top W ([t, m]^\top - [t_j, m_j]^\top)}
    dm dt \nonumber\\
    = &~\mu T |\mathcal{M}| + \frac{1}{D} \sum_{k=1}^D \sum_{i=0}^{N_T} \sum_{t_j < t_i} \int^{t_{i+1}}_{t_{i}} \int_\mathcal{M} e^{\boldsymbol{i}\omega_k^\top W ([t, m]^\top - [t_j, m_j]^\top)}
    dm dt\nonumber\\
    \label{eq:loglik-integral-1}
% \end{split}
\end{align}
% \end{equation}
\endgroup
Then the remainder of the proof calculates the integral $\int^{t_{i+1}}_{t_{i}} \int_\mathcal{M} e^{\boldsymbol{i}\omega_k^\top W ([t, m]^\top - [t_j, m_j]^\top)} dm dt$.
% given $\omega_k, t_i, t_{i+1}, t_j, m_j$. 
%
First, let the linear mapping matrix $W = [w_0|w_1|\dots|w_d]$ be split into $d+1$ column vectors, where $w_0 \in \mathbb{R}^{r \times 1}$, $w_{\ell} \in \mathbb{R}^{r \times 1}, \ell = 1,\dots,d$ correspond to the linear mappings for the time and mark subspace, respectively.
Denote the matrix formed by first $\ell$ column vectors of matrix $W$ as $W_{1:\ell} \coloneqq [w_1|\dots|w_\ell]$.
Denote the $\ell$-th dimension of $\mathcal{M}$ as $\mathcal{M}_\ell \in \mathbb{R}, \ell = 1,\dots,d$.
Assume each dimension of the mark space $\mathcal{M}_\ell, \ell = 1,\dots,d$ is normalized to range $[a, b]$. 
Denote the sub-space of $\mathcal{M}$ with the first $\ell$ dimensions as $\mathcal{M}_{1:\ell} \coloneqq \mathcal{M}_1 \times, \dots, \times \mathcal{M}_\ell, \ell = 1,\dots,d$. 
Denote the mark vector with first $\ell$ elements as $m_{1:\ell} = [m_1, \dots, m_\ell]$.
Then the integral can be written as
\begingroup
\allowdisplaybreaks
\begin{align}
    \int^{t_{i+1}}_{t_{i}} \int_\mathcal{M} 
    e^{\boldsymbol{i}\omega_k^\top W ([t, m]^\top - [t_j, m_j]^\top)} dm dt\nonumber
    = &~e^{- \boldsymbol{i}\omega_k^\top W [t_j, m_j]^\top} \int^{t_{i+1}}_{t_{i}} \int_\mathcal{M} 
    e^{\boldsymbol{i}\omega_k^\top W [t, m]^\top} dm dt\nonumber\\
    = &~e^{- \boldsymbol{i}\omega_k^\top W [t_j, m_j]^\top} 
    \int^{t_{i+1}}_{t_{i}} e^{\boldsymbol{i}\omega_k^\top w_0 t} dt
    \int_\mathcal{M} e^{\boldsymbol{i}\omega_k^\top W_{1:d} m^\top} dm.
    \label{eq:loglik-integral-2}
\end{align}
\endgroup
% To avoid notational overload, Define $e^{- \boldsymbol{i}\omega_k^\top W [t_j, m_j]^\top}$ as $E_{k,j}$. Let $f_{k,0}(t)$ denote $e^{\boldsymbol{i}\omega_k^\top w_0 t}$, $f_{k,\ell}(m)$ denote $e^{\boldsymbol{i} \omega_k^\top w_\ell m_\ell}$, and $F_{k,\ell}(m)$ denote $e^{\boldsymbol{i}\omega_k^\top W_{1:\ell} m_{1:\ell}^\top}$. 
% Note that $F_{k,\ell}(m) = f_{k,\ell}(m) F_{k,\ell-1}(m)$.
% Then the above equation can be written as 
To avoid notational overload, let $f_{k,\ell}(m)$ denote $e^{\boldsymbol{i} \omega_k^\top w_\ell m_\ell}$, and $F_{k,\ell}(m)$ denote $e^{\boldsymbol{i}\omega_k^\top W_{1:\ell} m_{1:\ell}^\top}$. 
Note that 
\[
    \int_a^b f_{k,\ell}(m) d m_\ell = \frac{e^{\boldsymbol{i}\omega_k^\top w_\ell b} - e^{\boldsymbol{i}\omega_k^\top w_\ell a}}{\boldsymbol{i}\omega_k^\top w_\ell},
\]
and $F_{k,\ell}(m) = f_{k,\ell}(m) F_{k,\ell-1}(m)$.
Then $\int_\mathcal{M} e^{\boldsymbol{i}\omega_k^\top W_{1:d} m^\top} dm$ can be written as 
\begin{align}
    \int_\mathcal{M} e^{\boldsymbol{i}\omega_k^\top W_{1:d} m^\top} dm 
    = &~\int_\mathcal{M} F_{k,d}(m) dm
    = \int_{\mathcal{M}_{1:d-1}} F_{k,d-1}(m) d m_{1:d-1} \int_a^b f_{k,d}(m) d m_d\nonumber\\
    = &~\prod_{\ell=1}^d \left( \int_a^b f_{k,\ell}(m) d m_\ell \right)
    = \prod_{\ell=1}^d \left( \frac{e^{\boldsymbol{i}\omega_k^\top w_\ell b} - e^{\boldsymbol{i}\omega_k^\top w_\ell a}}{\boldsymbol{i}\omega_k^\top w_\ell} \right).
    \label{eq:loglik-integral-3}
\end{align}
Substitute \eqref{eq:loglik-integral-3} into \eqref{eq:loglik-integral-2}, we have
\begingroup
\allowdisplaybreaks
\begin{align}
    &~e^{- \boldsymbol{i}\omega_k^\top W [t_j, m_j]^\top} 
    \int^{t_{i+1}}_{t_{i}} e^{\boldsymbol{i}\omega_k^\top w_0 t} dt
    \int_\mathcal{M} e^{\boldsymbol{i}\omega_k^\top W_{1:d} m^\top} dm\nonumber\\
    = &~e^{- \boldsymbol{i}\omega_k^\top W [t_j, m_j]^\top} 
    \left(\frac{e^{\boldsymbol{i}\omega_k^\top w_0 t_{i+1}} - e^{\boldsymbol{i}\omega_k^\top w_0 t_i}}{\boldsymbol{i}\omega_k^\top w_0}\right) \prod_{\ell=1}^d \left(\frac{e^{\boldsymbol{i}\omega_k^\top w_\ell b} - e^{\boldsymbol{i}\omega_k^\top w_\ell a}}{\boldsymbol{i}\omega_k^\top w_\ell}\right). \nonumber
    % \label{eq:loglik-integral-2}
\end{align}
\endgroup
Due to the fact that, for any $b > a$ where $a, b \in \mathbb{R}$,
\begingroup
\allowdisplaybreaks
\begin{align*}
    \frac{e^{\boldsymbol{i}\omega_k^\top w_\ell b} - e^{\boldsymbol{i}\omega_k^\top w_\ell a}}{\boldsymbol{i}\omega_k^\top w_\ell}
    = &~\frac{e^{\omega_k^\top w_\ell}}{\omega_k^\top w_\ell} \cdot 
    \frac{e^{\boldsymbol{i} b} - e^{\boldsymbol{i} a}}{\boldsymbol{i}} \\
    = &~\frac{e^{\omega_k^\top w_\ell}}{\omega_k^\top w_\ell} \cdot 
    e^{\boldsymbol{i} \frac{b + a}{2}} \cdot
    \frac{e^{\boldsymbol{i} \frac{b - a}{2}} - e^{\boldsymbol{i} \frac{a - b}{2}}}{\boldsymbol{i}}\\
    \overset{(i)}{=} &~\frac{2 e^{\omega_k^\top w_\ell}}{\omega_k^\top w_\ell}  
    \left( \cos\left(\frac{b + a}{2}\right) + \boldsymbol{i} \sin\left(\frac{b + a}{2}\right) \right)
    \sin\left(\frac{b-a}{2}\right)\\
    \overset{(ii)}{=} &~\frac{2 e^{\omega_k^\top w_\ell}}{\omega_k^\top w_\ell}  
    \cos\left(\frac{b + a}{2}\right) \sin\left(\frac{b - a}{2}\right).
\end{align*}
\endgroup
The equality $(i)$ holds because of Euler's formula. The equality $(ii)$ holds, since both $K$ and $p_\omega$ are real-valued, it suffices to consider only the real portion. Let $x_j$ denote $[t_j, m_j]^\top$ and substitute $\exp\{- \boldsymbol{i}\omega_k^\top W [t_j, m_j]^\top\}$ with $\cos(- \omega_k^\top W x_j)$. Thus, the integral~\eqref{eq:loglik-integral-2} can be written as
\begin{equation}
% \begin{aligned}
    \cos\left(-\omega_k^\top W x_j \right)
    \cos\left(\frac{t_{i+1} + t_{i}}{2}\right) \sin\left(\frac{t_{i+1} - t_{i}}{2}\right) \cos^d\left(\frac{b + a}{2}\right) \sin^d\left(\frac{b-a}{2}\right)
    \prod_{\ell=1}^{d+1} \frac{2 e^{\omega_k^\top w_\ell}}{\omega_k^\top w_\ell}.
    %
    % =&~\cos\left(-\omega_k^\top W x_j \right)
    % \cos\left(\frac{t_{i+1} + t_{i}}{2}\right) \sin\left(\frac{t_{i+1} - t_{i}}{2}\right)
    % \cos^d\left(\frac{b + a}{2}\right) \cdot \sin^d\left(\frac{b-a}{2}\right)
    % e^{\omega_k^\top W} \prod_{\ell=1}^d \frac{1}{\omega_k^\top w_\ell}.
    \label{eq:loglik-integral-4}
% \end{aligned}
\end{equation}
Finally, combining previous results in~\eqref{eq:loglik-integral-1}, \eqref{eq:loglik-integral-2}, and \eqref{eq:loglik-integral-4} gives the result in Proposition~\ref{prop:loglik-integral}, i.e.,
\begin{align*}
    &~\int_{\mathcal{X}} \lambda(x|\mathcal{H}_t;\theta_0) dx 
    = \mu T (b-a)^d + \frac{1}{D} \sum_{k=1}^D \sum_{i=0}^{N_T} \sum_{t_j < t_i} \\
    &~\cos\left(-\omega_k^\top W x_j \right)
    \cos\left(\frac{t_{i+1} + t_{i}}{2}\right) \sin\left(\frac{t_{i+1} - t_{i}}{2}\right) \cos^d\left(\frac{b + a}{2}\right) \sin^d\left(\frac{b-a}{2}\right)
    \prod_{\ell=1}^{d+1} \frac{2 e^{\omega_k^\top w_\ell}}{\omega_k^\top w_\ell}.
\end{align*}

\end{document}